\documentclass[acmsmall,screen,nonacm]{acmart}

\setcopyright{none}
\makeatletter
\@ifundefined{hyxmp@parse@acmart}{}{\let\hyxmp@parse@acmart\relax}
\makeatother

\usepackage{amsmath}
\usepackage{mathtools}
\usepackage{booktabs}
\usepackage{tabularx}
\usepackage{array}
\usepackage{multirow}
\usepackage{enumitem}
\usepackage{graphicx}
\usepackage{xcolor}
\usepackage{tikz}
\usepackage{url}
\usepackage{microtype}
\usepackage{algorithm}
\usepackage{algpseudocode}
\usepackage{placeins}
\usetikzlibrary{arrows.meta,positioning,fit,calc,shapes.geometric}

\hypersetup{
  hidelinks,
  pdftitle={Runtime Authorization for Resources Acquired by AI Agents},
  pdfauthor={Genliang Zhu and Chu Wang},
  pdfsubject={Provenance-bounded runtime authorization for resources acquired by AI agents},
  pdfcreator={LaTeX}
}

\setlist[itemize]{leftmargin=1.25em}
\setlist[enumerate]{leftmargin=1.45em}
\newcolumntype{L}[1]{>{\raggedright\arraybackslash}p{#1}}
\newcolumntype{Y}{>{\raggedright\arraybackslash}X}

\definecolor{opBlue}{HTML}{3B82F6}
\definecolor{opBlueFill}{HTML}{EFF6FF}
\definecolor{opBlueText}{HTML}{1E40AF}
\definecolor{opGreen}{HTML}{22C55E}
\definecolor{opGreenFill}{HTML}{F0FDF4}
\definecolor{opGreenText}{HTML}{15803D}
\definecolor{opAmber}{HTML}{FCD34D}
\definecolor{opAmberFill}{HTML}{FEF3C7}
\definecolor{opAmberText}{HTML}{92400E}
\definecolor{opRed}{HTML}{EF4444}
\definecolor{opRedFill}{HTML}{FEF2F2}
\definecolor{opRedText}{HTML}{DC2626}
\definecolor{opPurple}{HTML}{8B5CF6}
\definecolor{opPurpleFill}{HTML}{F5F3FF}
\definecolor{opPurpleText}{HTML}{7C3AED}

\newtheorem{definition}{Definition}
\newtheorem{lemma}{Lemma}
\newtheorem{theorem}{Theorem}
\newtheorem{proposition}{Proposition}

\newcommand{\mcode}[1]{\ifmmode\text{\normalfont\ttfamily #1}\else{\normalfont\ttfamily #1}\fi}
\newcommand{\sem}[1]{\mathopen{[\![}#1\mathclose{]\!]}}
\newcommand{\restrict}{\mathbin{\sqsubseteq}}
\newcommand{\Episodes}{\ensuremath{\mathcal U}}
\newcommand{\Caps}{\ensuremath{\mathcal C}}

\newcommand{\Grants}{\ensuremath{\mathcal G}}

\newcommand{\Project}{\ensuremath{\Pi}}

\title[Runtime Authorization for Acquired Resources]
{Runtime Authorization for Resources Acquired by AI Agents}

\author{Genliang Zhu}
\affiliation{%
  \institution{Accentrust}
  \city{Vancouver}
  \country{Canada}
}
\affiliation{%
  \institution{Georgia Institute of Technology}
  \city{Atlanta}
  \state{Georgia}
  \country{USA}
}
\email{research@accentrust.com}

\author{Chu Wang}
\affiliation{%
  \institution{Accentrust}
  \city{Vancouver}
  \country{Canada}
}
\affiliation{%
  \institution{University of Illinois Urbana-Champaign}
  \city{Urbana}
  \state{Illinois}
  \country{USA}
}
\email{wally@accentrust.com}

\begin{abstract}
By acquiring compute, credentials, accounts, services, and other agents,
autonomous AI agents can introduce new authority into a task. Payment, budget, OAuth,
mandate, and fulfillment checks can validate transaction conditions without
deciding whether a returned resource may become usable authority. This
post-fulfillment activation gap spans tool-mediated creation, inter-agent
delegation, and agentic commerce.

We present a provenance-bounded runtime authorization architecture. It
quarantines acquired outputs, resolves their actual
capabilities from authenticated provider evidence through a versioned resolver,
and activates them only through a current activation transaction that checks
the resolved manifest, provenance, epochs, and a downward-closed relational
envelope over a typed resource--capability hypergraph. The envelope preserves
correlated identity, effect, data, delegation, and graph-wide limits. Single-use
effect permits are revalidated and consumed at effect linearization.

Under explicit assumptions, we prove eight safety properties covering
quarantine, backing, non-amplification, split non-evasion, crash/retry, refunds,
epochs, and effect confinement. Across five resource classes, reference
semantics accepted 20/20 benign traces and rejected 40/40 registered unsafe
traces over 810 events; an independent checker agreed on 60 base and 40
refinement traces and rejected 89/89 tamper tests. Frozen Codex and Gemini Model
Context Protocol (MCP) client components completed 54/54 deterministic local
stdio calls. In a registered 18-case staged MCP-to-Docker composition, both
benign paths completed, and none of the 16 unsafe paths added an unauthorized
Docker start request. A five-source audit classified 1,248 field pairs across
32 units; no unit alone supplied a complete activation profile.
\end{abstract}


\begin{document}
\maketitle
\hypersetup{
  pdftitle={Runtime Authorization for Resources Acquired by AI Agents},
  pdfauthor={Genliang Zhu and Chu Wang},
  pdfsubject={Provenance-bounded runtime authorization for resources acquired by AI agents},
  pdfcreator={LaTeX}
}

\section{Introduction}
\label{sec:introduction}

Autonomous agents increasingly act through interfaces that can create the
means for later action. A successful call may return an API credential,
activate a hosted service, launch a compute environment, register an account,
or establish a new remote agent. Public evaluations already treat obtaining
resources and deploying onto compute as components of autonomous replication
\cite{black2025replibench}, while current risk frameworks track long-range
autonomy and autonomous replication and adaptation as distinct research
categories \cite{openai2025preparedness}. These developments turn resource
acquisition into an access-control event.

The security boundary is easy to place at the wrong transition. A payment
protocol can faithfully bind an order and payment; a budget monitor can prove
that the price remains below a limit; OAuth can authorize the provisioning
endpoint; and a tool manifest can remain unchanged. None of these facts alone
establishes that the resource returned by the provider may become executable
authority for the task. The same authorized cloud-create call can yield an
instance with a task-local read role or one attached to a broader service
identity. The transaction record can be identical while the resulting action
surface differs.

Cloud platforms already recognize provider-specific forms of this distinction.
AWS separates permission to launch an EC2 instance from permission to pass a
particular IAM role, and Google Cloud requires service-account impersonation
permission when an identity is attached to a resource
\cite{aws2026passrole,google2026actas}. Our question is broader and orthogonal:
what authorization semantics lets an autonomous-agent control plane govern
the returned resource consistently across providers and resource classes,
including zero-price enrollment, credentials, remote services, and new agent
principals?

\subsection{The post-fulfillment activation gap}

We distinguish six events:

\[
\textit{propose}\rightarrow\textit{reserve}\rightarrow\textit{dispatch}
\rightarrow\textit{quarantine}\rightarrow\textit{resolve}
\rightarrow\textit{activate}.
\]

Proposal authorizes an intended acquisition. Reservation protects any
consumable source authority. Dispatch starts an external operation.
Quarantine records a returned resource without exposing usable authority.
Resolution derives the resource's actual semantic capability from registered
provider evidence. Activation is a new authorization decision over the
resulting capability and the complete current acquisition graph.

This separation yields a strict witness. Suppose a task may spend twenty
synthetic dollars on one isolated read-only compute instance. Its payment,
amount, endpoint scope, order mandate, and delivery receipt all validate. The
provider state nevertheless resolves to an administrator credential with
external-network effects. Every transaction-side predicate can remain true;
the activation predicate must be false. A second witness costs zero: four
aliases independently enroll four free workers when the envelope permits one
active descendant. A third witness refunds a purchase without destroying its
credential, restoring money but not authority capacity. These cases establish
that the missing object is neither a scalar budget nor the authorization of
the acquisition call.

\subsection{Approach}

We present a reference architecture for provenance-bounded
activation. Its acquisition envelope is a downward-closed relational predicate
over a normalized typed hypergraph. Each active hyperedge may combine multiple
inputs and produces one ordinal-indexed output; a multi-output acquisition is
represented by distinct, independently admitted singleton-output edges. Nodes
represent root grants, source-resource
references, quarantined resources, and active semantic capabilities. The
predicate can express correlated limits that componentwise checks lose: one
beneficiary may hold a read capability or a publish capability but not both;
one canonical control root may create at most one active descendant; or a
credential may be valid only for a particular data domain, provider profile,
purpose, and delegation depth. All grants sharing an issuer, canonical control
root, and policy epoch are evaluated in one authority-domain aggregate, so a
fresh grant identifier cannot reset those limits.

Before external dispatch, durable source reservations and an immutable
acquisition permit bind the source-only inputs, request, expected descriptor,
provider profile, purpose, current state root and epochs, and a root-qualified
provider operation key. The dispatch outbox carries that exact permit; a
provider receipt must correlate to it before any returned output can enter
quarantine.

An external output first enters a broker-controlled vault. A versioned
resolver maps authenticated provider state to a normalized actual-capability
manifest. An activation transaction then validates the conversion relation,
the complete prospective graph, policy and identity epochs, and a single-use
permit. Each output carries an authenticated ordinal inside the registered
finite output range and therefore occupies its own activation slot. Opaque
handle publication then revalidates the exact committed active row and its
currentness; a fenced or expired row is never published for use. Every
protected use rechecks the current lineage and requested effect, then binds
them into a single-use effect permit that is revalidated while its slot is
consumed at the
effect linearization point.
The permit also binds the identity, version, and digest of a finite
target-to-data-domain registry, so an unknown or incorrectly mapped target
cannot borrow an allowed data-domain label.
Payment success and delivery are evidence inputs, not implicit activation
grants.

\subsection{Contributions}

This paper makes five contributions:

\begin{enumerate}
  \item It identifies and formalizes the \emph{post-fulfillment activation
  gap}. An indistinguishability proposition proves that a monitor restricted
  to transaction-side evidence cannot decide activation soundly when equal
  transaction records may produce different semantic outputs.

  \item It defines a typed resource--capability acquisition hypergraph and a
  relational, downward-closed envelope over an authority-domain-normalized
  active projection. The model retains identity, co-possession, cardinality,
  delegation, and graph constraints rather than reducing authority to
  independent fields or per-grant counters.

  \item It gives a quarantine-before-activation protocol with version-bound
  output resolution, canonical identity, single-use activation, refund-safe
  lineage, epoch fencing, and an immutable single-use effect permit whose slot is consumed at the
  protected effect's linearization point.

  \item It establishes eight conditional safety results: quarantine
  non-authority, backed activation, acquisition non-amplification, split
  non-evasion, crash-safe at-most-once activation, refund non-resurrection,
  epoch non-inheritance, and end-to-end resource-to-effect confinement.

  \item It evaluates the separation with 60 five-class base fixtures over 810
  events, 40 implementation-refinement traces, a checker that rejects all 89
  registered tamper tests, two separate agent-runtime mappings, AP2 and
  external-source profiles, 15 activation and five effect crash cuts, the
  complete 32-schedule registered five-bit replay domain, 32 effect replays,
  192 receipt substitutions, local kernel timing, and a 20-case
  network-disabled container gate. Two frozen upstream MCP client
  components complete 27 calls each against byte-identical deterministic
  stdio witnesses. An 18-case composition admits one server-observed call per
  logical case to a staged path; each advances through capture binding and the
  common IR only until its registered rejection boundary or authorized
  completion. Twelve cases reach Docker quarantine and four authorized first
  starts. The evidence records keep conditional proof, finite checking, and
  deployment premises distinct.
\end{enumerate}

\subsection{Scope}

The model covers registered, decidable acquisition profiles whose outputs,
identity aliases, and protected effects are completely mediated. It does not
claim to decide arbitrary right-acquisition safety; that problem is
undecidable in the general protection model \cite{harrison1976protection}.
Ordinary tool installation is outside scope unless it creates an external
account, credential, tenant, provider resource, or principal. General agent
planning, model alignment, covert offline resources, and unobservable
real-world consequences are not proof premises. Within the stated profile,
however, an active capability must have a current and non-amplifying
derivation; missing or indeterminate evidence cannot authorize activation.

The remainder of the paper separates the problem from adjacent controls
(Section~\ref{sec:background}), states the system and assurance model
(Section~\ref{sec:system}), formalizes the semantics and results
(Section~\ref{sec:formal}), presents the design (Section~\ref{sec:design}),
defines protocol instantiations and evaluation evidence
(Sections~\ref{sec:instantiations}--\ref{sec:evaluation}), and closes with
related work and the exact assurance boundary.

\section{Background and Problem Separation}
\label{sec:background}

\subsection{Authorization before and after acquisition}

Classical least privilege asks that a subject receive only the authority
needed for its task \cite{saltzer1975protection}. ABAC evaluates attributes of
subjects, objects, actions, and environments \cite{nist800162}; usage control
extends authorization across mutable attributes and ongoing use
\cite{park2004ucon}; and zero-trust architecture places decisions near each
resource rather than treating network position as trust \cite{nist800207}.
The proposed architecture applies these principles to a transition that changes the set of
resources and principals available to an agent task.

Table~\ref{tab:layers} separates five decisions that can coexist in one
workflow. Each existing layer establishes its own decision object;
the proposed authorization layer adds the distinct decision over post-fulfillment activation.

\begin{table*}[t]
\caption{Complementary decisions in an acquisition workflow.}
\label{tab:layers}
\small
\begin{tabularx}{\textwidth}{L{0.16\textwidth} L{0.22\textwidth} Y Y}
\toprule
Layer & Primary decision object & Representative evidence & Question not settled by that object alone \\
\midrule
Endpoint authorization & proposed API operation & caller, audience, scope, request parameters & What semantic authority does the returned resource create? \\
Commerce & represented order and transfer & intent, checkout, amount, payee, signed mandates, receipt & May the fulfilled item become usable task authority? \\
Resource accounting & source-side consumable rights & balance, reservation, ownership, settlement & Does a zero-price or nonnumeric output widen future effects? \\
Capability/tool binding & an existing named capability or tool definition & manifest, version, grant, invocation handle & What new resource instance, identity, or entitlement did this call create? \\
Acquisition activation & actual returned resource and prospective active graph & provider state, resolved manifest, canonical identity, lineage, epochs & Is this output admissible now, and what may it do later? \\
\bottomrule
\end{tabularx}
\end{table*}

\subsection{Commerce and payment protocols}

AP2 defines signed checkout and payment mandates and receipts for
agent-mediated commerce \cite{ap22026spec}. UCP covers commerce capability discovery,
checkout, order lifecycle, identity linking, and payment exchange
\cite{ucp2026}. Recent work formalizes agent-payment protocol lifecycles and
their cross-stage relations \cite{jiang2026formalpayments}, analyzes AP2 v0.2
threat surfaces \cite{aviv2026beyond}, and adds runtime replay and context
binding to mandate use \cite{lan2026zerotrustpayments}. Our architecture consumes
valid transaction and fulfillment records from such systems. Its decisive
input is the separately resolved authority of the delivered resource.

This distinction matters even when fulfillment is correct. A merchant can
deliver exactly the requested cloud tier while the tier's default identity,
network, credential exportability, or derived-principal behavior exceeds the
task envelope. Conversely, a mandate may contain a custom constraint whose
meaning is sufficient for a registered activation profile. The adapter may use that
authenticated field; the architecture does not presume that mandates are
inexpressive.

\subsection{OAuth, agent protocols, and tool execution}

OAuth resource indicators enable authorization servers to audience-restrict
access tokens, token exchange represents delegation and impersonation contexts,
and Rich Authorization Requests carry structured authorization details
\cite{rfc8707,rfc8693,rfc9396}. Current OAuth
security guidance remains part of the authentication and token-handling
baseline \cite{rfc9700}. MCP's current authorization profile uses OAuth-based
mechanisms, while its tool interface represents model-invocable operations
\cite{mcp2026authorization,mcp2026tools}. A2A 1.0 represents remote agents,
tasks, and authentication requirements \cite{a2a2026spec}.

Precise task-scoped authorization, attenuated delegation, and principal-chain
models govern who may invoke an existing action and under which task envelope
\cite{sharma2026pauth,prakash2026aip,muruaga2026bounded}. HCP makes canonical
resources, principal binding, grant-backed invocation, and handle provenance
explicit in MCP-style execution \cite{liu2026executioncontrol}. These controls
are direct inputs to operation-time enforcement. Deterministic pre-action
authorization and deployed multi-layer agent access control likewise evaluate
request parameters, agent identity, and execution context before an existing
operation runs \cite{uchibeke2026oap,malik2026granular}. The authorization layer asks the
subsequent admission question when such an authorized invocation returns a new
resource or principal whose authority was not yet in the active graph.

Governance work also separates an agent's technical capability level from its
allowed autonomy level \cite{zheng2026autonomylevels}. Our architecture gives this
distinction a resource-instance transition semantics: the declared operation
may be allowed while the actual returned capability remains quarantined and
unauthorized.

\subsection{Dynamic capabilities and acquisition}

Dynamic-capability governance binds tool manifests and interaction evidence
to cryptographic identities and detects post-authorization tool changes
\cite{zhou2026dynamic}. Capability negotiation protocols cover discovery,
selection, negotiation, attestation, and binding between heterogeneous agents
\cite{huang2025acnbp}. These mechanisms govern dynamic declarations and
bindings. The proposed architecture governs a distinct transition over a provider-created
resource instance: the tool manifest may remain exactly constant while a call
through that tool creates a credential, instance, account, entitlement, or new
principal with distinct semantic authority.

\subsection{Why a relational graph is necessary}

A numeric vector can constrain money, calls, compute units, or effect counts;
resource-bounded agent contracts make such dimensions explicit
\cite{ye2026agentcontracts}. A capability descriptor can attenuate effects.
Neither representation alone captures all acquisition policies. Consider an
envelope that permits a canonical beneficiary to hold either a sensitive-read
capability or an external-publish capability, but not both. Each capability's
fields independently belong to an allowed set, and the union of individually
allowed fields may admit the forbidden pair. Similarly, ``at most one active
descendant'' is a graph-cardinality condition, not an effect-set inclusion.

We therefore define the safety envelope as a predicate over a
normalized acquisition hypergraph. The predicate is downward closed under
removing nodes or edges, narrowing semantic capabilities, reducing counts, and
shortening validity. Task-completion requirements are evaluated separately;
they need not be downward closed. This separation keeps the safety theorem
precise without equating denial with successful task execution.

\section{System, Threat, and Assurance Model}
\label{sec:system}

\subsection{Actors and enforcement boundary}

Figure~\ref{fig:architecture} places activation between fulfillment and use.
The \emph{grant issuer} signs a root acquisition envelope. The agent runtime
proposes acquisitions. A source authority service supplies reservation and
settlement evidence for consumable inputs. An external provider creates a
resource. The acquisition broker receives that output into quarantine. A
profile-specific resolver reads authoritative provider state and emits an
actual-capability manifest. The activation monitor evaluates the prospective
graph and publishes only a brokered opaque handle. The effect broker mediates
later protected uses.

\begin{figure*}[t]
\centering
\resizebox{0.98\textwidth}{!}{%
\begin{tikzpicture}[
  >=Latex,
  node distance=7mm and 8mm,
  box/.style={draw, rounded corners, align=center, minimum height=9mm,
              inner xsep=6pt, font=\small},
  blue/.style={box, draw=opBlue, fill=opBlueFill, text=opBlueText},
  purple/.style={box, draw=opPurple, fill=opPurpleFill, text=opPurpleText},
  green/.style={box, draw=opGreen, fill=opGreenFill, text=opGreenText},
  amber/.style={box, draw=opAmber, fill=opAmberFill, text=opAmberText},
  red/.style={box, draw=opRed, fill=opRedFill, text=opRedText},
  flow/.style={->, thick}
]
\node[blue] (grant) {signed acquisition\\envelope $\Gamma$};
\node[blue, right=of grant] (proposal) {normalized proposal\\and source evidence};
\node[purple, right=of proposal] (provider) {external acquisition\\and fulfillment};
\node[amber, right=of provider] (vault) {quarantine vault\\opaque resource};
\node[green, right=of vault] (resolve) {actual-capability\\resolver};
\node[green, below=10mm of resolve] (activate) {relational graph gate\\atomic activation};
\node[green, left=of activate] (handle) {brokered active\\opaque handle};
\node[blue, left=of handle] (use) {current-closure gate\\single-use permit};
\node[green, left=of use] (commit) {effect-linearization\\commit};
\node[red, below=7mm of activate] (denyactivate) {deny or\\indeterminate};
\node[red, below=7mm of use] (denyuse) {deny or fence};

\draw[flow, opBlue] (grant) -- (proposal);
\draw[flow, opPurple] (proposal) -- (provider);
\draw[flow, opAmberText] (provider) -- (vault);
\draw[flow, opGreenText] (vault) -- (resolve);
\draw[flow, opGreenText] (resolve) -- (activate);
\draw[flow, opGreenText] (activate) -- (handle);
\draw[flow, opBlue] (handle) -- (use);
\draw[flow, opGreenText] (use) -- (commit);
\draw[flow, opRed] (activate) -- (denyactivate);
\draw[flow, opRed] (use) -- (denyuse);
\end{tikzpicture}%
}
\caption{The architecture separates successful acquisition from executable
authority. Provider output remains quarantined until its actual manifest,
canonical identity, lineage, and prospective aggregate satisfy the current
authority domain's joint envelope. Later effects re-enter the current-closure
gate and commit only by
atomically redeeming a current single-use effect permit.}
\label{fig:architecture}
\Description{A left-to-right acquisition pipeline places provider fulfillment
before a quarantine vault and actual-capability resolver. A relational graph
gate creates an active opaque handle only after validation. Every protected use
passes through a current-closure gate, and the effect commit atomically
revalidates the immutable single-use permit and consumes its slot; invalid or stale paths are
denied or fenced.}
\end{figure*}
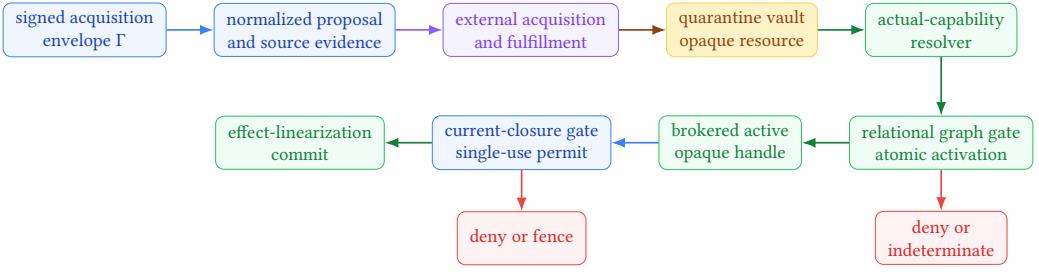

The agent never receives raw provider credentials in the brokered profile. It
receives a non-exportable handle whose invocation path terminates at the effect
broker. If a provider interface cannot keep the raw authority outside the
agent, a conforming deployment requires a remote gateway that consumes the
activation permit before the resource becomes usable. An interface that
supports neither containment pattern is not an activatable profile.

\subsection{Adversary}

The adversary controls the agent's planning and may coordinate all descendants
under the same authority domain. It may choose valid calls, legal payment routes, free
tiers, discounts, refunds, account aliases, concurrent orders, retry timing,
and credential transfers. It may replay messages, exploit crashes, present a
stale profile, substitute a requested field for an actual field, or attempt to
use a resource after a policy, identity, provider, or revocation epoch changes.

The adversary cannot forge a trusted signature or digest, compromise the
reference monitor, break the broker's quarantine, or produce a protected
effect outside every declared mediation boundary. Providers may return a
stronger or differently configured resource than requested. For a registered
profile, however, the provider evidence fields used by the resolver are
authenticated according to that profile. Missing or ambiguous evidence causes
an indeterminate decision, never an inferred grant.

\subsection{Assurance assumptions}

Table~\ref{tab:assumptions} states the complete theorem premises. They are
versioned profile requirements and checker inputs, not informal expectations.

\begin{table*}[t]
\caption{Assurance assumptions.}
\label{tab:assumptions}
\small
\begin{tabularx}{\textwidth}{L{0.06\textwidth} L{0.25\textwidth} Y}
\toprule
ID & Assumption & Operational obligation \\
\midrule
A1 & Complete mediation & Every registered acquisition, provider output, activation, handle publication, and protected effect crosses the declared broker. \\
A2 & Authentic state & Grants, contracts, provider evidence, identity mappings, target--domain maps, epochs, and commitments are unforgeable within the profile. \\
A3 & Conservative resolution & The actual-capability projector is total and complete for policy-relevant fields; unknowns map to a sound upper bound or cause rejection. \\
A4 & Canonical identity & Controlled subjects, beneficiaries, resources, and aliases resolve to authoritative canonical roots; ambiguity fails closed. \\
A5 & Correct policy semantics & Capability denotation, conversion relations, graph normalization, and envelope predicates implement their registered mathematical definitions. \\
A6 & Transactional durability & Each authority-changing operation executes as one serializable storage transaction. A transaction either aborts without changing durable state or commits its complete write set at one linearization point; a successful commit is durable before success is returned. The store enforces durable \(\mathsf{UNIQUE}(\mathit{slot})\) constraints on \(UsedActivationKeys\), \(Active\), \(UsedEffectKeys\), and \(EffectReceipts\); a conflict aborts. Recovery exposes the state after the last complete commit and no incomplete write. Lifecycle transactions include every required dependent fence. A handle is derived deterministically from, and may be published only for, a committed active record. \\
A7 & Recoverable provider operation & The broker derives a domain-separated provider key from the root-qualified logical operation; the provider authenticates that exact key in its receipt and supports an authoritative terminal-state query. An unresolved or uncorrelated operation remains quarantined. \\
A8 & Current epochs & Policy, identity, contract, resource, grant, and revocation epochs are authoritatively readable at activation and use. \\
A9 & Source evidence contract & Consumable input reservation and settlement evidence satisfies its registered adapter contract; the activation monitor does not reconstruct a source ledger. \\
A10 & Effect linearization & Protected downstream operations are normalized to semantic episodes and can occur only at a registered commit point that re-resolves current lineage and atomically validates the exact effect permit, records its single-use slot consumption, and accepts the effect. An executor without this primitive is outside the conforming profile. \\
A11 & Finite decidable profile & The registered conversion closure and relational predicate terminate; the base transition system rejects every cross-root input combination. \\
\bottomrule
\end{tabularx}
\end{table*}

\subsection{Assets and attack goals}

The protected asset is the mediated task's exercisable authority, not merely
its inventory. The adversary succeeds if it causes any of the following:

\begin{itemize}
  \item an executable resource without a current path to an acquisition root;
  \item a graph that violates an effect, purpose, beneficiary, provider,
  multiplicity, delegation, replication, or co-possession clause;
  \item a handle published before its durable activation record;
  \item two active authorities created from one activation slot or logical
  acquisition output;
  \item a refund or restored source budget treated as destruction of a still
  live capability; or
  \item a protected effect accepted after its capability, ancestry, manifest,
  or policy state becomes stale.
\end{itemize}

Unmediated offline resources, hidden provider behavior outside the registered
evidence profile, arbitrary social engineering, side channels, and
unmodellable physical consequences are outside the theorem domain. The monitor
does guarantee fail-closed activation when required in-domain evidence is
absent, inconsistent, stale, or exceeds the envelope.

\section{Formal Model}
\label{sec:formal}

This section defines the registered acquisition profile, its reachable states,
and the safety properties enforced by activation. Appendix~\ref{app:proofs}
gives the proof details.

\subsection{Episodes and semantic capabilities}

Let \(\mathcal S\) be the universe of model states and \(\Episodes\) the
universe of normalized protected-effect episodes. An episode records every
field consumed by operation-time policy:

\[
u=\langle \mathit{episodeId},q,b,\mathit{eff},t,d,p,r,c,\tau\rangle .
\]

Here \(\mathit{episodeId}\) is the registered semantic episode identifier,
\(q\) is the invoking subject, \(b\) the canonical beneficiary,
\(\mathit{eff}\) an effect class, \(t\) a target, \(d\) a data domain, \(p\) a purpose,
\(r\) a recipient or audience, \(c:\mathcal S\to\{\mathsf{true},
\mathsf{false}\}\) a registered decidable state predicate, and \(\tau\) a
nonempty logical-time interval.  If \(\mathsf{now}(s)\) denotes the logical
time recorded in state \(s\), define
\[
 \mathsf{EpisodeCurrent}_s(u)\equiv
 c(s)\land \mathsf{now}(s)\in\tau .
\]
Each effect profile has a commitment
\(m=\langle\mathit{mapId},\mathit{mapVersion},\mathit{mapDigest}\rangle\)
to a finite target-to-domain map. Normalization is defined only when
\(t\) has exactly one registered image and \(d=\mathsf{domain}_{m}(t)\);
target identity and domain identity need not be equal. An unknown target,
ambiguous image, stale map version, or digest mismatch is rejected before a
permit is issued. Profiles may add typed fields; omission of a policy-consumed
field is not a wildcard. Both preparation and commit evaluate the same
canonical schema and separately bind the same map commitment \(m\), while
\(c\) and \(\tau\) are evaluated
again against the commit state.

\begin{definition}[Capability descriptor]
\label{def:capability}
A capability descriptor \(\alpha\in\Caps\) records an episode set and
structural constraints:
\[
 \alpha=\langle \mathit{aid},\mathcal E_\alpha,\delta_\alpha,
 \ell_\alpha,n_\alpha,\mu_\alpha,\nu_\alpha,
 \upsilon_\alpha,o_\alpha\rangle,
 \qquad \sem{\alpha}=\mathcal E_\alpha\subseteq\Episodes .
\]
Here \(\mathit{aid}\) is a stable descriptor identifier,
\(\delta_\alpha\) is redelegability, \(\ell_\alpha\) a maximum delegation
depth, \(n_\alpha\) a descendant bound, \(\mu_\alpha\) a concurrency bound,
\(\nu_\alpha\) a multiplicity bound, \(\upsilon_\alpha\) the descriptor
validity interval, and \(o_\alpha\) its continuing obligations.  Each
registered obligation has a decidable satisfaction relation
\(\mathsf{Obl}_s(o,u)\).  Define the obligation preorder by
\[
 o_1\preceq_{\mathsf O}o_2
 \quad\Longleftrightarrow\quad
 \forall s\in\mathcal S,\,u\in\Episodes:\
 \mathsf{Obl}_s(o_1,u)\Rightarrow\mathsf{Obl}_s(o_2,u).
\]
The descriptor restriction order is
\[
\begin{aligned}
 \alpha_1\restrict\alpha_2\quad\Longleftrightarrow\quad{}
 &\mathcal E_{\alpha_1}\subseteq\mathcal E_{\alpha_2}
 \land (\delta_{\alpha_1}\Rightarrow\delta_{\alpha_2})
 \land \ell_{\alpha_1}\leq\ell_{\alpha_2}
 \land n_{\alpha_1}\leq n_{\alpha_2}\\
 &\land \mu_{\alpha_1}\leq\mu_{\alpha_2}
 \land \nu_{\alpha_1}\leq\nu_{\alpha_2}
 \land \upsilon_{\alpha_1}\subseteq\upsilon_{\alpha_2}
 \land o_{\alpha_1}\preceq_{\mathsf O}o_{\alpha_2}.
\end{aligned}
\]
Registered canonicalization assigns the same content-derived \(\mathit{aid}\)
to descriptors with mutually equivalent obligations, episode denotations, and
structural fields. The relation is therefore an antisymmetric partial order on
canonical descriptors (equivalently, on their semantic equivalence classes).
Thus every structural dimension of a restricted descriptor is explicitly no
wider.
\end{definition}

\begin{definition}[Actual capability instance]
\label{def:instance}
An acquired capability instance is the immutable tuple
\[
 x=\langle \mathit{xid},\mathit{rid},k,\pi,\mathit{cr},b,
          \alpha_x,h_x,\eta_x\rangle,
\]
where \(\mathit{xid}\) is its broker-assigned canonical instance identifier,
\(\mathit{rid}\) its canonical provider resource identifier, \(k\) its
resource kind, \(\pi\) its provider, \(\mathit{cr}\) an
authoritative controller root, \(b\) a canonical beneficiary, \(h_x\) the
actual-manifest root, and \(\eta_x\) its epoch vector.  Mutable commerce,
delivery, and authority coordinates belong to the state-indexed operation map
\(\sigma_s\) defined below; changing them never changes the identity of \(x\).
Surface account names, order labels, email aliases, and agent-provided identity
strings are not \(\mathit{cr}\) or \(b\).
\end{definition}

Possessing an identifier for \(x\) does not make it executable.  The exact
state predicates \(\mathsf{HandleEnabled}_s(x)\), \(\mathsf{Exec}_s(x)\), and
\(\mathsf{AcceptedEffect}_s(x,u)\) are defined with the machine state below;
in particular, neither a vault identifier nor a prepared permit satisfies any
of them.

\subsection{Relational acquisition envelopes}

An acquisition envelope must retain whole-policy correlations. Independent
allowlists for effects, domains, and subjects are insufficient because they
implicitly authorize their Cartesian product.

\begin{definition}[Acquisition envelope]
\label{def:envelope}
A root acquisition envelope is
\[
 \Gamma=\langle \mathit{gid},i,\mathit{cr}_\Gamma,\mathit{task}_\Gamma,
 K_\Gamma,\Phi_\Gamma,
 S_\Gamma,P_\Gamma,T_\Gamma,V_\Gamma,
 E_\Gamma,\epsilon_\Gamma,\eta_\Gamma\rangle .
\]
It binds an issuer \(i\), canonical control root \(\mathit{cr}_\Gamma\), task or
session identifier \(\mathit{task}_\Gamma\), registered contract set
\(K_\Gamma\), a safety predicate \(\Phi_\Gamma\), authorized
source references \(S_\Gamma\), provider profiles \(P_\Gamma\), purposes
\(T_\Gamma\), validity \(V_\Gamma\), admitted evidence-profile set
\(E_\Gamma\), registered source-reservation evidence profile
\(\epsilon_\Gamma\), and an
epoch map \(\eta_\Gamma\). Let \(\mathsf{pe}(\Gamma)\) be the authoritative policy
epoch named in \(\eta_\Gamma\). The grant identifier preserves provenance; it
is not an aggregation boundary.
\end{definition}

\begin{definition}[Authority domain]
\label{def:authority-domain}
The authority domain of a grant is
\[
 \rho(\Gamma)=\langle i,\mathsf{canon}(\mathit{cr}_\Gamma),
 \mathsf{pe}(\Gamma)\rangle .
\]
Current membership is the state predicate
\[
\begin{aligned}
 \mathsf{CurrentGrant}_s(\Gamma)\equiv{}
 &\mathsf{VerifyIssuer}_s(\Gamma)
 \land \mathsf{now}(s)\in V_\Gamma\\
 &\land\ \eta_\Gamma=
 \mathsf{Epochs}_s\!\restriction\mathsf{dom}(\eta_\Gamma).
\end{aligned}
\]
Let
\[
 \mathsf{Live}_\rho(s)=
 \{\Gamma\in\mathit{Grants}_s\mid
   \rho(\Gamma)=\rho\land\mathsf{CurrentGrant}_s(\Gamma)\}.
\]
Its joint safety decision is the conjunction
\[
 \mathsf{Safe}_\rho(s,H)\equiv
 \bigwedge_{\Gamma\in\mathsf{Live}_\rho(s)}\Phi_\Gamma(H).
\]
Thus two fresh \(\mathit{gid}\) values under the same issuer, canonical control root,
and policy epoch share one aggregate. A new policy epoch creates a new domain
only after the old domain has been fenced under A8.
\end{definition}

Let \(H_1\preceq_s H_2\) mean that \(H_1\) can be obtained from \(H_2\) by
removing active acquisition nodes or edges, narrowing capability descriptors,
reducing counts or concurrency, shortening validity, disabling delegation, or
removing co-possessed capability nodes and thereby eliminating their
relation---without changing retained canonical identities, lineage, or
contract types.

\begin{definition}[Downward-closed safety predicate]
\label{def:downward}
For normalized active acquisition graphs, each \(\Phi_\Gamma\) is downward closed
when
\[
 \Phi_\Gamma(H_2)\land H_1\preceq_s H_2
 \Longrightarrow \Phi_\Gamma(H_1).
\]
Registration additionally requires \(\Phi_\Gamma(\varnothing)=\mathsf{true}\);
a candidate root whose empty projection is unsafe is not installed.
Because a conjunction of downward-closed predicates is downward closed,
\(\mathsf{Safe}_\rho\) has the same property for every live grant set.
Registration also requires authority-domain extensionality: if two normalized
graphs differ only by renaming or redistributing provenance \(\mathit{gid}\) labels
among grants in the same \(\rho\), every \(\Phi_\Gamma\) in that domain returns
the same value on them. Consequently, a policy may inspect canonical
identities and graph relations, but not use a fresh grant identifier as new
capacity.
Minimum utility, task-completion, and availability conditions are evaluated by
a separate goal predicate and do not belong to \(\Phi_\Gamma\).
\end{definition}

Examples expressible in a domain's joint predicate include ``at most one
active worker under this control root,'' ``read and external publish may not be co-possessed
by one beneficiary,'' ``no output is redelegable,'' and ``all paths to this
data domain use provider profile \(v\).''

\subsection{Registered conversions and the acquisition hypergraph}

\begin{definition}[Conversion contract]
\label{def:contract}
A registered conversion contract is
\[
 \kappa=\langle \mathit{kid},v,I_\kappa,O_\kappa,N_\kappa,
 R_\kappa,F_\kappa,E_\kappa,J_\kappa,\eta_\kappa\rangle .
\]
\(I_\kappa\) and \(O_\kappa\) are allowed input and output types.  Let
\(\mathcal Z_\kappa\) be the authenticated provider-state space,
\(\mathcal H\) the manifest-digest space, and \(\mathcal F_\kappa\) the
registered set of policy-consumed provider fields.  Resolution has the typed
total signature
\[
 N_\kappa:\mathcal Z_\kappa\longrightarrow
 \bigl(\Caps\times\mathcal H\times 2^{\mathcal F_\kappa}\bigr)
 \uplus\{\bot\}.
\]
If \(N_\kappa(z)=(\alpha,h,U)\), then \(U\) is the set of required fields
not observed exactly, and \(h\) commits to the authenticated evidence, the
resolved fields, \(U\), every substituted upper bound, and \(\alpha\).
For each \(f\in U\), the profile supplies a field order \(\preceq_f\) and a
registered bound \(\mathsf{ub}_{\kappa,f}(z)\) satisfying
\[
 \forall z'\equiv_{E_\kappa}z:\quad
 \mathsf{value}_f(z')\preceq_f\mathsf{ub}_{\kappa,f}(z),
 \tag{UB}\label{eq:resolver-upper-bound}
\]
where \(z'\equiv_{E_\kappa}z\) means consistency with the same authenticated
evidence.  The bound is incorporated into \(\alpha\).  If any required
unresolved field lacks such a sound bound, the only permitted result is
\(\bot\); unresolved-but-bounded fields remain explicit in \(U\).
Let \(\mathcal I_\kappa\) and \(\mathcal O_\kappa\) be the finite, well-typed
input and output tuples admitted by \(I_\kappa\) and \(O_\kappa\), and let
\(\Omega_\kappa\) be the finite registered output-ordinal set. The relation
\(R_\kappa:\mathcal I_\kappa\times\mathcal O_\kappa\to
\{\mathsf{true},\mathsf{false}\}\) specifies conversion validity, and
\[
 F_\kappa:\mathcal I_\kappa\times\Grants\times\Omega_\kappa
 \longrightarrow\Caps
\]
returns the complete registered descriptor upper bound for output ordinal
\(\omega\) derivable from inputs \(I\) under \(\Gamma\).
\(E_\kappa\) is the evidence profile;
\(J_\kappa\) the identity profile; and \(\eta_\kappa\) the contract epoch.
\end{definition}

The acquisition state is a directed hypergraph \(G_s=(V_s,E_s)\). An active
edge
\[
 e:I_e\xrightarrow{\kappa_e,\omega_e,\Gamma_e}\{x\}
\]
may consume several source or capability nodes and has one active output.
A provider result with several resources is represented by distinct fresh
quarantined nodes with registered output ordinals; the base activation profile
uses one singleton-output edge per node. Each edge binds the root grant,
logical operation key and output ordinal, canonical input and output
identities, provider receipt, actual-manifest root, contract version, control
root, beneficiary, and all policy-relevant epochs. After canonical identity
resolution, the graph is acyclic.

The base profile gives every active node exactly one root \(\mathit{gid}\) for
provenance and rejects every conversion whose capability inputs are derived
from different roots; the monitor never unions derivations implicitly. A
separate join profile would require a registered import transition that binds
and continuously revalidates every predecessor row and root, and is not part
of this transition system. This single-root provenance rule does not partition
policy accounting: all roots in the same authority domain remain in one
aggregate projection.

\begin{definition}[Normalized active projection]
\label{def:projection}
Write \(\Project_\Gamma(s)\) for the root-specific active subgraph used for
provenance and per-edge reasoning. The policy projection is
\[
 \Project_\rho(s)=\mathsf{erase}_{\mathit{gid}}\!\left(
   \bigcup_{\Gamma\in \mathsf{Live}_\rho(s)}\Project_\Gamma(s)\right).
\]
It removes surface aliases, batch boundaries, order numbers, retry
identifiers, and grant-identifier labels, while retaining actual capability
descriptors, conversion types, canonical beneficiaries and control roots,
counts, delegation, multiplicity, co-possession, cross-boundary lineage
references, and graph structure. Immutable root provenance remains in the
separate audit projection \(\Project_\Gamma\); erasing its label only prevents
policy accounting from treating a new label as new capacity.
\end{definition}

For a finite, well-typed set \(E\) of fresh candidate edges rooted in domain
\(\rho\), define the candidate projection and normalized graph merge by
\[
 \Project^{+}_{\rho}(s,E)=\mathsf{normalize}_{\rho}\!\left(
   \Project_{\rho}(s)\cup E\cup\mathsf{out}(E)\right),
 \qquad
 H\boxplus_{\rho}\Delta=\mathsf{normalize}_{\rho}(H\cup\Delta).
\]
For a singleton we write \(\Project^{+}_{\rho}(s,e)\). The operator admits
only fresh canonical output identities, registered ordinals, and well-typed
edges; a collision, type error, or cycle is undefined and therefore rejected.
It adds candidate authority solely for admission evaluation: it does not add
a durable row, consume a slot, or publish a handle. If a singleton commit of
\(e\) takes \(s\) to \(s'\), its transactional write set must establish
\(\Project_{\rho}(s')=\Project^{+}_{\rho}(s,e)\).

Normalization intentionally prevents a policy from distinguishing two plans
only because one uses more order or grant IDs. It never collapses two actual
outputs or two independently exercisable authorities. In particular, issuing
\(\Gamma_1\) and \(\Gamma_2\) with distinct \(\mathit{gid}\) but equal \(\rho\) cannot
reset a cardinality or co-possession constraint.
For example, suppose both grants' joint policy permits at most one active
child-agent. After one child-agent rooted at \(\Gamma_1\) is active, a
candidate rooted at \(\Gamma_2\) makes the shared projected count two and is
rejected. Issuing \(\Gamma_2\) therefore cannot reset the bound.

\subsection{Machine state and authenticated objects}

A state is
\[
 \begin{aligned}
s=\langle&\theta,\mathit{Grants},G,\mathit{Ops},
 \mathit{AcquisitionPermits},\mathit{Outbox},\Sigma,
 \mathit{Vault},\mathit{Resolved},\mathit{ActivationPermits},\\
 &\mathit{Active},\mathit{Handles},\mathit{EffectPermits},
 \mathit{EffectReceipts},\mathit{Epochs},\mathit{RootRevisions},\\
 &\mathit{UsedActivationKeys},\mathit{UsedEffectKeys}\rangle .
 \end{aligned}
\]
The logical clock is part of durable state: \(\mathsf{now}(s)=\theta_s\), and
it changes only through \(\mathsf{AdvanceTime}\). For a raw logical operation
key \(\mathit{opKey}\), registered output ordinal \(\omega\), and root
\(\Gamma\), define the root-qualified key
\(o=\langle\mathit{gid}_\Gamma,\mathit{opKey}\rangle\) and activation slot
\(j=\langle\mathit{gid}_\Gamma,\mathit{opKey},\omega\rangle_{\mathsf{act}}\).
Write \(\mathsf{rawKey}(o)=\mathit{opKey}\), and derive the provider-facing
idempotency key
\[
 \mathsf{providerKey}(o)=
 \mathsf{Com}_{\mathsf{provider\mbox{-}operation}}(o).
\]
An unqualified provider, request, or order identifier is never a state key.
The immutable operation record at that key is
\[
 \widehat{o}=\mathit{Ops}_s(o)=
 \langle\Gamma_o,\kappa_o,v_o,\pi_o,I_o,\mathit{task}_o,\mathit{cr}_o,b_o,
          \mathit{expected}_o,\mathit{purpose}_o,\eta_o,\mathit{pk}_o\rangle,
\]
where \(\pi_o\) is the registered provider profile and every component is the
canonical value admitted at proposal, with
\(\mathit{pk}_o=\mathsf{providerKey}(o)\). The output ordinal is deliberately not
an operation-record field: each received output carries its own
\(\omega\in\Omega_{\kappa_o}\). Later transitions resolve \(\widehat{o}\) by
the root-qualified key and reject an absent or unequal record.
\(\mathit{AcquisitionPermits}\) is an immutable partial map on the same keys;
proposal installs its entry atomically with \(\widehat{o}\).
\(\mathit{Outbox}_s(o)\) is the durable dispatch row containing the exact
operation record, acquisition permit, source reservation evidence,
root-bound provider idempotency key, and dispatch status. It is installed before any
provider send and is the sole source for retry after recovery. The
state-indexed output map
\[
 \sigma_s(o,\omega)=\Sigma_s(o,\omega)
   =(C_s(o),D_s(o,\omega),A_s(j))
\]
separates operation-level commerce from output-level delivery and slot-level
authority:
\begin{align*}
 C_s(o) &\in \{\mcode{none},\mcode{proposed},\mcode{reserved},
 \mcode{submitted},\\
     &\qquad \mcode{accepted},\mcode{rejected},\mcode{indeterminate},
 \mcode{settled},\mcode{refunded}\},\\
 D_s(o,\omega) &\in \{\mcode{absent},\mcode{quarantined},\mcode{resolved},
 \mcode{retired}\},\\
 A_s(j) &\in \{\mcode{none},\mcode{active},
 \mcode{fenced},\mcode{revoked}\}.
\end{align*}
Preparation appends \(p_x\) but leaves \(A_s(j)=\mcode{none}\);
\(\mathsf{CommitActivation}\) is the single transition from
\mcode{none} to \mcode{active}.
Thus outputs of one provider operation may independently be quarantined,
active, or fenced while sharing one commerce outcome. For example,
\((\mcode{settled},\mcode{resolved},\mcode{none})\) at one output coordinate
records a paid and delivered resource whose activation was denied; and
\((\mcode{refunded},\mcode{resolved},\mcode{active})\) records a refund while
that output's capability remains live. \(\mathit{Vault}\) is inaccessible to the agent;
\(\mathit{Handles}\) contains only brokered references; and the two used-key sets
record durably consumed activation and effect slots. \(\mathit{RootRevisions}\) is
indexed by authority domain, not by \(\mathit{gid}\): it is the revision of the
canonical controller-root domain shared by every grant in \(\rho\).

An active row has the canonical schema
\[
 a=\langle j,x,o,\omega,\chi,\alpha_x,h_x,e,\Gamma,
       \eta_a,v_a\rangle,
\]
where \(j\) is its activation slot and \(\chi\) its deterministic opaque
handle.  Write \(\mathsf{slot}(a)=j\), \(\mathsf{inst}(a)=x\),
\(\mathsf{op}(a)=o\), \(\mathsf{ordinal}(a)=\omega\),
\(\mathsf{handle}(a)=\chi\),
\(\mathsf{root}(a)=\Gamma\), and
\(\rho_a=\rho(\mathsf{root}(a))\).
Define the immutable handle input by removing the handle coordinate itself,
\[
 \mathsf{rowCore}(a)=
 \langle j,x,o,\omega,\alpha_x,h_x,e,\Gamma,\eta_a,v_a\rangle .
\]
Thus handle derivation is a function of committed row content and contains no
self-reference.
Every vault object carries an authenticated correlation accessor
\(\mathsf{corr}_s(x)=\langle o,\omega\rangle\).  For
\(o\in\mathsf{dom}(\mathit{Ops}_s)\) and
\(j=\langle\mathit{gid}_{\Gamma_o},\mathsf{rawKey}(o),\omega\rangle_{\mathsf{act}}\),
define activation freshness by
\[
\begin{aligned}
 \mathsf{ActivationFresh}_s(x,o,\omega)\equiv{}&
 x\in\mathit{Vault}_s
 \land \mathsf{corr}_s(x)=\langle o,\omega\rangle\\
 &{}\land \bigl(\forall x'\in\mathit{Vault}_s:\
       \mathsf{corr}_s(x')=\langle o,\omega\rangle\Rightarrow x'=x\bigr)\\
 &{}\land \mathit{Resolved}_s(x)\ne\bot
 \land D_s(o,\omega)=\mcode{resolved}
 \land \neg\exists e\in E_s:\ x\in O_e\\
 &{}\land \bigl(\neg\exists a\in\mathit{Active}_s:\
       \mathsf{inst}(a)=x\bigr)
 \land A_s(j)=\mcode{none}.
\end{aligned}
\]
This predicate means fresh for activation, not absent from state: the output is
already quarantined and resolved, but no graph edge, active row, or authority
slot has admitted it. The commit separately checks that the corresponding
single-use slot has not been consumed.
The last field is the \(v_s(s^-,o,x)\) reconstructed in the serialized
activation-commit pre-state \(s^-\); it is retained as immutable audit
evidence rather than treated as the current global state.
The predicate \(\mathsf{RowFor}_s(a,x)\) holds exactly when
\(a\in\mathit{Active}_s\), \(\mathsf{inst}(a)=x\), and every immutable
instance, root-qualified operation key, output ordinal, manifest, descriptor,
edge, root, and epoch field in \(a\) equals
the committed value for \(x\).  The predicate
\(\mathsf{LineageCurrent}_s(a)\) is defined inductively over the row's
acyclic committed sub-hypergraph. It holds exactly when that sub-hypergraph has
the live root \(\mathsf{root}(a)\), satisfies each edge's registered relation
and ordinal-specific upper bound, has matching canonical identities,
manifests, epochs, and validity intervals at every node, and every non-source
capability predecessor \(y\) has a matching durable row \(a_y\) with
\(A_s(\mathsf{slot}(a_y))=\mcode{active}\) and
\(\mathsf{CurrentRow}_s(a_y,y)\). Acyclicity makes this predecessor recursion
well founded. The predicate \(\mathsf{SlotConsumed}_s(a)\) holds exactly when
\(\mathsf{slot}(a)\in\mathit{UsedActivationKeys}_s\).
Define the row-level currentness predicate
\[
\begin{aligned}
 \mathsf{CurrentRow}_s(a,x)\equiv{}&
 \mathsf{RowFor}_s(a,x)
 \land \mathsf{SlotConsumed}_s(a)\\
 &{}\land \mathsf{root}(a)\in\mathsf{Live}_{\rho_a}(s)\\
 &{}\land \mathsf{now}(s)\in\upsilon_{\alpha_x}
 \land \eta_a=\mathsf{Epochs}_s\!\restriction\mathsf{dom}(\eta_a)
 \land \mathsf{LineageCurrent}_s(a).
\end{aligned}
\]
The predicate \(\mathsf{HandleEnabled}_s(x)\) holds exactly when
\[
 \exists!a\in\mathit{Active}_s:\
 \mathsf{CurrentRow}_s(a,x)
 \land A_s(\mathsf{slot}(a))=\mcode{active}
 \land\mathit{Handles}_s(\mathsf{handle}(a))=a.
\]
Finally,
\[
 \mathsf{Exec}_s(x)\equiv\mathsf{HandleEnabled}_s(x).
\]
A committed effect receipt records the exact instance, canonical episode,
effect permit, slot, protected-effect result, accepted decision, and an A10
linearization witness. \(\mathit{EffectPermits}\) is an append-only ledger of
immutable prepared objects; ``consumption'' means atomically adding the permit's
slot to \(\mathit{UsedEffectKeys}\), not deleting or rewriting the permit.
\(\mathit{EffectReceipts}\) is likewise append-only. Let
\(\mathsf{ValidAcceptanceWitness}(r)\) mean that the registered broker
verifies the receipt's A10 witness. With the field accessors of that canonical
receipt schema, define
\[
\begin{aligned}
 \mathsf{AcceptedEffect}_s(x,u)\equiv{}&
 \exists r,p_u:\quad r\in\mathit{EffectReceipts}_s\\
 &{}\land p_u\in\mathit{EffectPermits}_s
 \land\mathsf{inst}(r)=x\\
 &{}\land\mathsf{episode}(r)=u
 \land\mathsf{permit}(r)=p_u
 \land\mathsf{decision}(r)=\mcode{accepted}\\
 &{}\land\mathsf{ValidAcceptanceWitness}(r)
 \land\mathsf{receiptSlot}(r)=\mathsf{effectSlot}(p_u)\\
 &{}\land\mathsf{effectSlot}(p_u)\in\mathit{UsedEffectKeys}_s.
\end{aligned}
\]
The receipt, consumed slot, and protected effect are installed at the single
A10 linearization point. This is a historical acceptance predicate; I6 below
states what had to hold in the linearization pre-state.

Call a state \emph{refresh-closed} when (i) every \(a\in\mathit{Active}_s\)
with \(A_s(\mathsf{slot}(a))=\mcode{active}\) satisfies
\(\mathsf{CurrentRow}_s(a,\mathsf{inst}(a))\), whether or not its handle has
yet been published; and (ii) I5 holds. The deterministic operation
\(\mathsf{Refresh}\) monotonically fences every committed row with an active
slot-level authority coordinate that fails \(\mathsf{CurrentRow}\), together
with every dependent descendant. If a newly current grant makes an existing
domain projection violate I5, it fences that domain's active-authority rows; the
resulting empty projection is safe by
Definition~\ref{def:downward}.  It disables their handles, sets the relevant
authority coordinates to \mcode{fenced}, and advances every affected domain
revision in the same transaction.  It never creates authority.  Every
accepted state is refresh-closed, and no activation or effect decision may
observe an intermediate, unrefreshed clock, epoch, or grant-currentness
change.

For a canonical typed object \(z\), let \(\mathsf{Com}_d(z)\) be its
domain-separated, length-delimited collision-resistant commitment under A2,
and let \(\mathsf{Ver}_s(Q)\) be the A6 storage revision token for a named
serialized read set \(Q\), which changes whenever any row in \(Q\) changes.
For \(o\in\mathit{dom}(\mathit{Ops}_s)\), write
\(\rho_o=\rho(\Gamma_o)\). The security-critical commitments used
below are
\[
\begin{aligned}
h_a(o,\widehat{o})&=\mathsf{Com}_{\mathsf{acquire}}
   \bigl(o,\widehat{o}\bigr),\\
 h_\Gamma(\Gamma)&=\mathsf{Com}_{\mathsf{envelope}}(\Gamma),\\
 \mathsf{Src}(o,\widehat{o})&=
   \bigl\langle\mathit{gid}_{\Gamma_o},\mathsf{rawKey}(o),\epsilon_{\Gamma_o},
   \mathsf{sort}(\mathsf{sourceRefs}(I_o))\bigr\rangle,\\
 h_{\mathsf{src}}(o,\widehat{o})&=
   \mathsf{Com}_{\mathsf{source\mbox{-}evidence}}
   \bigl(\mathsf{Src}(o,\widehat{o})\bigr),\\
 h_e(e)&=\mathsf{Com}_{\mathsf{edge}}(e),\\
 h^+_\rho(s,e)&=\mathsf{Com}_{\mathsf{projection}}
   \bigl(\Project^+_\rho(s,e)\bigr),\\
 v_s(s,o,x)&=\mathsf{Ver}_s
   \bigl(\mathit{Ops}_s(o),\mathit{Resolved}_s(x),
         \Project_{\rho_o}(s),\mathit{Epochs}_s,
         \mathit{RootRevisions}_s(\rho_o)\bigr).
\end{aligned}
\]
Thus \(h_a\) and \(h_\Gamma\) are pure commitments that can be computed before
proposal atomically installs the operation; \(\mathsf{Src}\) is the exact
canonical statement attested under the envelope's registered
source-reservation evidence profile \(\epsilon_{\Gamma_o}\), distinct from the
provider/resolver evidence profile \(E_{\kappa_o}\), and
\(h_{\mathsf{src}}\) is its domain-separated digest carried by the durable
outbox. The
other tokens commit to named canonical rows rather than to an unspecified
ambient state. Preparation records these values; commit reconstructs them from
its serialized read set.

\begin{definition}[Acquisition permit]
\label{def:acquire-permit}
An acquisition permit binds
 \[
 p_a=\mathsf{Sig}_{i}\bigl(
 \mathit{gid},h_\Gamma(\Gamma),\mathit{task}_o,\mathit{opKey},\mathit{pk},\kappa,v,\pi,I,b,\mathit{cr},
 \mathit{expected},\mathit{purpose},
 \eta,h_a(o,\widehat{o}),\mathit{expiry}\bigr).
\]
Here \(\mathit{pk}=\mathsf{providerKey}(o)\); the serialized
\mcode{providerOperationKey} field carries that value.
The \mcode{grantEnvelopeDigest} field carries \(h_\Gamma(\Gamma)\).
The \mcode{providerProfileId} field carries \(\pi\)'s identifier; the
\mcode{providerProfileVersion} field carries its version.
The \mcode{taskId} field carries the operation's canonical
\(\mathit{task}_o=\mathit{task}_\Gamma\).
The serialized artifact field named \mcode{stateRoot} carries
\(h_a(o,\widehat{o})\).
It authorizes dispatch and quarantine under one logical operation. It does not
authorize execution of any returned resource.
\end{definition}

\begin{definition}[Activation permit]
\label{def:activation-permit}
After resolution, a single-use activation permit binds
\[
 \begin{aligned}
p_x=\mathsf{Sig}_{M}\bigl(&\mathit{xid},\mathit{rid},h_x,
 \alpha_x,h_e(e),\\
&\mathit{gid},\rho(\Gamma),\mathit{rootRevision},\mathit{opKey},\omega,
 h^+_{\rho(\Gamma)}(s,e),\\
&\mathit{cr},b,h_\Gamma(\Gamma),\pi_o,h_{\mathsf{src}}(o,\widehat{o}),
 \kappa,v,\eta,v_s(s,o,x),\mathit{expiry}\bigr),
\end{aligned}
\]
where \(M\) is the activation monitor. Exact equality is required for every
bound digest and identity at commit. The artifact fields
\mcode{canonicalResourceId}, \mcode{edgeRoot},
\mcode{prospectiveGraphRoot}, and \mcode{stateVersion} carry, respectively,
\(\mathit{rid}\), \(h_e(e)\), \(h^+_{\rho(\Gamma)}(s,e)\), and
\(v_s(s,o,x)\). The \mcode{grantEnvelopeDigest} and
\mcode{sourceEvidenceDigest} fields carry \(h_\Gamma(\Gamma)\) and
\(h_{\mathsf{src}}(o,\widehat{o})\), respectively. The
\mcode{providerProfileId} field carries \(\pi_o\)'s identifier; the
\mcode{providerProfileVersion} field carries its version. The contract derives \(\omega\) from the
authenticated provider result using a registered injective, retry-stable
output-key map; ambiguous or conflicting associations fail closed. Define the
durable activation slot by
\[
\mathsf{slot}(p_x)=
 \langle\mathit{gid},\mathit{opKey},\omega\rangle_{\mathsf{act}}.
\]
The model slot is this structured tuple. An implementation may represent it
by a domain-separated, length-delimited collision-resistant hash commitment
under A2; it must first validate the canonical encoding and treat a detected
collision or digest-verification failure as rejection.
Retries and duplicate receipts for the same logical output therefore address
the same slot even if their transport identifiers differ.
\(\mathit{ActivationPermits}\) is an append-only ledger of these immutable
prepared objects. Activation consumption adds the permit's slot to
\(\mathit{UsedActivationKeys}\); it never deletes or rewrites \(p_x\).
\end{definition}

\begin{definition}[Effect permit]
\label{def:effect-permit}
After a current use check, the effect broker prepares
 \[
 p_u=\mathsf{Sig}_{B}(\mathit{uid},u,\chi,\mathit{activeRoot},\rho,m,
 \eta,\mathit{rootRevision},\mathit{effectOpKey},\mathit{expiry}),
\]
where \(B\) is the effect broker, \(u\) is the complete normalized episode,
\(\mathit{activeRoot}\) commits to the durable active record, \(\rho\) is the active
root's authority domain, and \(\mathit{rootRevision}\) is that domain's current value
in \(\mathit{RootRevisions}\). The separately signed \(m\) binds the
target-to-domain map's identifier, version, and digest to the resolved target
and domain in \(u\). Its single-use slot is
\[
 \mathsf{effectSlot}(p_u)=
 \langle\rho,\mathit{effectOpKey}\rangle_{\mathsf{eff}}.
\]
As with activation slots, a domain-separated collision-resistant hash is an
implementation representation of this structured key, not an assumption that
an arbitrary hash function is injective.
Here \(\mathit{uid}\) identifies the permit object; it is not a second episode
identifier. The semantic \(\mathit{episodeId}\) occurs exactly once, inside
the canonical \(u\). The logical \(\mathit{effectOpKey}\) is an authenticated idempotency key. Neither
preparation nor possession of \(p_u\) is an effect; only a successful
\(\mathsf{CommitEffect}\) under A10 records effect acceptance.
\end{definition}

An acquisition receipt proves only the statement defined by its registered
provider profile. It is not substituted for \(h_x\) or \(\alpha_x\).

\subsection{Transitions}

For grant issuance only, \(s\oplus\Gamma\) denotes the hypothetical state
whose grant set is \(\mathit{Grants}_s\cup\{\Gamma\}\) and whose other
components equal those of \(s\); it is inspected before any write commits.
The accepted transition relation contains the following guarded operations.

\begin{enumerate}
  \item \(\mathsf{IssueGrant}\) verifies the signature and registered profile,
  requires
  \(\nexists\Gamma'\in\mathit{Grants}_s:
    \mathit{gid}_{\Gamma'}=\mathit{gid}_{\Gamma}\) and
  \(\Phi_\Gamma(\varnothing)=\mathsf{true}\), and forms
  \(s^+=s\oplus\Gamma\).  If \(\mathsf{CurrentGrant}_{s^+}(\Gamma)\), it also
  requires
  \(\mathsf{Safe}_{\rho(\Gamma)}(s^+,
  \Project_{\rho(\Gamma)}(s^+))\).  It then installs \(\Gamma\) and advances
  the domain revision.  A not-yet-current grant contributes no authority;
  its later validity onset is handled by the atomic revalidation in
  \(\mathsf{AdvanceTime}\). A fresh \(\mathit{gid}\) in an existing domain
  neither creates an empty projection nor resets capacity. An existing grant
  is never overwritten; policy-epoch replacement first fences the old domain
  under A8.
  \item \(\mathsf{AdvanceTime}(\theta')\) requires
  \(\theta'\geq\mathsf{now}(s)\), writes \(\theta'\), and applies
  \(\mathsf{Refresh}\) in the same serialized transaction.  Grant onset,
  grant or descriptor expiry, and every time-dependent registered predicate
  are therefore reflected before the resulting state is observable.  If a
  newly live conjunct rejects a nonempty domain projection, the domain is
  fenced and its revision advances; no accepted step observes the unsafe
  intermediate state.
  \item \(\mathsf{ProposeAcquire}\) requires
  \(\Gamma\in\mathsf{Live}_{\rho(\Gamma)}(s)\), \(\kappa\in K_\Gamma\),
  \(\mathsf{sourceRefs}(I_o)\subseteq S_\Gamma\),
  \(\pi_o\in P_\Gamma\), \(\mathit{purpose}_o\in T_\Gamma\), and
  \(E_\kappa\in E_\Gamma\), and
  \(\mathit{task}_o=\mathit{task}_\Gamma\); every capability input must be executable with a
  current derivation to that same \(\Gamma\). It also validates the beneficiary and
  invoking subject against the expected output and validates every expected
  output clause. It constructs
  the canonical \(\widehat{o}\) above and \(p_a\) containing
  \(h_a(o,\widehat{o})\), then, in successor \(s'\), atomically
  installs \(\mathit{Ops}_{s'}(o)=\widehat{o}\) together with
  \(\mathit{AcquisitionPermits}_{s'}(o)=p_a\), but creates no capability. If
  \(\mathit{Ops}_s(o)\) already exists, an exact match is an idempotent retry
  that returns the original record and durably stored permit; any unequal
  field or missing paired permit rejects.
  \item \(\mathsf{ReserveSource}\) binds authenticated source-side reservation
  evidence under \(\epsilon_{\Gamma_o}\) to the root-qualified
  \(o=\langle\mathit{gid},\mathit{opKey}\rangle\).
  \item \(\mathsf{DispatchAcquire}\) requires
  \(p_a=\mathit{AcquisitionPermits}_s(o)\), verifies its signature and exact
  operation, grant, \(h_a(o,\mathit{Ops}_s(o))\), source, input-lineage,
  identity, and epoch bindings, requires
  the bound grant to remain in \(\mathsf{Live}_{\rho(\Gamma_o)}(s)\), every
  bound epoch to be current, and
  \(\mathsf{now}(s)<\mathsf{expiry}(p_a)\), and persists an outbox
  record indexed by \(o\) before sending \(\mathsf{providerKey}(o)\).
  \item \(\mathsf{AcquireIntoQuarantine}\) requires authenticated provider
  evidence that echoes the exact \(\mathsf{providerKey}(o)\), receipt identity,
  profile, operation, and registered output ordinal; it then correlates the result,
  assigns each returned resource its registered output ordinal, and creates a
  fresh vault node \(x\) with \(\mathsf{corr}_{s'}(x)=\langle o,\omega\rangle\)
  and \(D_{s'}(o,\omega)=\mcode{quarantined}\); it publishes no handle. An exact
  replay returns the existing correlated node without a state change. If that
  root-qualified operation and ordinal already name a different canonical
  resource or manifest, the association conflicts and the transition rejects.
  \item \(\mathsf{ResolveActualCapability}\) obtains authenticated provider
  state \(z\) and applies \(N_\kappa\).  An authenticated-correlation mismatch
  rejects without a state change; if \(N_\kappa(z)=\bot\), the transition makes
  no state change and returns indeterminate. Otherwise it atomically verifies
  the output's authenticated \(\langle o,\omega\rangle\) correlation, stores the
  exact typed result
  \(\mathit{Resolved}_{s'}(x)=(\alpha_x,h_x,U_x)\), including every sound bound
  used for an unresolved field, and sets
  \(D_{s'}(o,\omega)=\mcode{resolved}\) in the successor state \(s'\).
  \item \(\mathsf{PrepareActivation}\) constructs a candidate singleton-output
  hyperedge. It verifies that every input is either an authenticated source
  reference authorized by the root or a current executable capability with a
  derivation to that same root,
  and verifies \(\Gamma\in\mathsf{Live}_{\rho(\Gamma)}(s)\),
  \(\kappa\in K_\Gamma\), \(\omega_e\in\Omega_\kappa\), well-typed inputs and
  output, \(\mathsf{ActivationFresh}_s(x,o,\omega_e)\), \(R_\kappa\),
  \(\alpha_x\restrict F_\kappa(I_e,\Gamma,\omega_e)\), all epochs,
  descriptor validity, acyclicity, and
  \begin{equation}
    \mathsf{Safe}_{\rho(\Gamma)}\bigl(
      s,\Project^{+}_{\rho(\Gamma)}(s,e)\bigr),
    \label{eq:prospective-phi}
  \end{equation}
  then appends and emits one \(p_x\).
  Equation~\eqref{eq:prospective-phi} is the prospective shared-domain guard:
  it is evaluated on the complete candidate projection rather than the
  pre-activation or root-local graph.
  \item \(\mathsf{CommitActivation}\) atomically reconstructs the candidate
  edge and prospective projection from current durable state, then repeats the
  live-root, \(\kappa\in K_\Gamma\), \(\omega_e\in\Omega_\kappa\), input and
  output typing, input/root, activation freshness, \(R_\kappa\), output-ordinal upper-bound,
  identity, manifest, grant-envelope, provider-profile, source-evidence,
  descriptor-validity, acyclicity, epoch,
  domain-revision, and joint-predicate checks. It requires
  \(p_x\in\mathit{ActivationPermits}_s\),
  \(\mathsf{now}(s)<\mathsf{expiry}(p_x)\), and
  \(\mathsf{slot}(p_x)\notin\mathit{UsedActivationKeys}_s\) and proposes the
  graph append, used-slot consumption row, and one slot-keyed active
  record as one transaction write set. Under A6 the complete set commits at
  one durable linearization point, or a unique-slot conflict aborts it without
  change; a successful commit also advances the domain revision.
  \item \(\mathsf{PublishHandle}\) resolves the exact durable row \(a\) and
  its deterministic opaque reference
  \(\chi=\mathsf{Opaque}(\mathsf{slot}(a),\mathsf{rowCore}(a))\), requires
  \(\mathsf{CurrentRow}_s(a,\mathsf{inst}(a))\) and
  \(A_s(\mathsf{slot}(a))=\mcode{active}\), and requires
  \(\mathit{Handles}_s(\chi)\in\{\bot,a\}\). It then atomically installs
  \(\mathit{Handles}_{s'}(\chi)=a\) and exposes \(\chi\), leaving every other
  component unchanged. The existing equal mapping is an idempotent retry; a
  collision with another row fails closed.
  \item \(\mathsf{PrepareEffect}\) resolves the submitted handle \(\chi\) to
  one exact active row \(a\) for instance \(x\), sets
  \(\rho=\rho(\mathsf{root}(a))\), requires
  \(\mathit{Handles}_s(\chi)=a\),
  \(A_s(\mathsf{slot}(a))=\mcode{active}\), and
  \(\mathsf{CurrentRow}_s(a,x)\) and \(\mathsf{Exec}_s(x)\), normalizes the
  complete requested episode \(u\), and requires
  \(u\in\sem{\alpha_x}\), \(\mathsf{EpisodeCurrent}_s(u)\),
  \(\mathsf{Obl}_s(o_{\alpha_x},u)\), and
  \(\mathsf{Safe}_{\rho}(s,\Project_\rho(s))\), then appends and emits the
  exactly bound immutable single-use permit in
  Definition~\ref{def:effect-permit}.
  \item \(\mathsf{CommitEffect}\) executes at the protected effect's
  linearization point. It re-resolves the handle to the same row \(a\) and
  instance \(x\), sets \(\rho=\rho(\mathsf{root}(a))\), and re-resolves the
  domain, complete epoch vector, domain revision, map commitment, and episode;
  requires exact equality
  with \(p_u\), \(\mathit{Handles}_s(\chi)=a\),
  \(A_s(\mathsf{slot}(a))=\mcode{active}\),
  \(\mathsf{CurrentRow}_s(a,x)\), \(\mathsf{Exec}_s(x)\),
  \(\mathsf{now}(s)<\mathsf{expiry}(p_u)\), \(u\in\sem{\alpha_x}\),
  \(\mathsf{EpisodeCurrent}_s(u)\),
  \(\mathsf{Obl}_s(o_{\alpha_x},u)\), the joint predicate, and an unused
  effect slot, with \(p_u\in\mathit{EffectPermits}_s\); then atomically appends
  the used-slot consumption row, accepted
  receipt, and protected-effect result as required by A10. A
  mismatch, concurrent time advance, fence,
  or replay aborts before the effect.
  \item \(\mathsf{Fence}\) monotonically advances the relevant epoch and
  invokes \(\mathsf{Refresh}\), which disables old handles, atomically fences
  every committed active-authority row (including an unpublished row) and
  descendant whose registered derivation ceases to be current or valid,
  restores I5 if a live-grant set changes, and advances each affected domain
  revision.
  \item \(\mathsf{RefundSettled}\) updates source evidence without removing
  the capability node or releasing a semantic capacity clause.
  \item \(\mathsf{DestroyVerified}\) releases an instance's active-capacity
  contribution only after authoritative evidence proves the same canonical
  resource cannot be used. It disables the handle and changes that slot's
  authority coordinate to \mcode{revoked}, while retaining immutable
  provenance. Any descendant whose registered derivation requires the resource
  is atomically fenced before the instance leaves the active projection, and
  the transition advances the affected domain revision.
  \item \(\mathsf{ReconcileLateResult}\) sends a late success to quarantine;
  it never invokes activation implicitly.
\end{enumerate}

\subsection{Inductive safety properties}

Every reachable accepted state satisfies I1--I5, and every accepted trace
satisfies the trace-indexed property I6:

\begin{description}[leftmargin=3.2em,style=nextline]
  \item[I1 --- Exact actual-output binding.]
  Every prepared or committed activation is bound to the actual manifest,
  canonical resource, controller, beneficiary, grant envelope, versioned
  provider profile, authenticated source-evidence digest, contract version and
  output ordinal, prospective graph, security-state version, and complete epoch vector in
  Definition~\ref{def:activation-permit}. Every prepared effect permit is
  exactly bound to the single canonical episode schema in
  Definition~\ref{def:effect-permit}.

  \item[I2 --- Quarantine first.]
  \[
   \mathsf{Exec}_s(x)\Rightarrow
   \exists!a\in\mathit{Active}_s:\
   \mathsf{RowFor}_s(a,x)\land
   A_s(\mathsf{slot}(a))=\mcode{active}.
  \]
  Created, received, or resolved status is insufficient.

  \item[I3 --- Single-root provenance and current acyclic derivation.]
  Every executable instance has exactly one current root and at least one
  acyclic derivation path whose grants, contracts, identities, resource
  versions, validity intervals, and epochs all verify.  Its root
  \(\Gamma\) belongs to \(\mathsf{Live}_{\rho(\Gamma)}(s)\). Its policy
  contribution is nevertheless aggregated with every live root in the same
  authority domain.

  \item[I4 --- Valid conversion edge.]
  For each active edge \(e\) with committed root \(\Gamma_e\),
  \[
   \kappa_e\in K_{\Gamma_e}\ \land\
   \omega_e\in\Omega_{\kappa_e}\ \land\
   I_e\in\mathcal I_{\kappa_e}\ \land\
   O_e\in\mathcal O_{\kappa_e}\ \land\
   R_{\kappa_e}(I_e,O_e)\ \land\
   \forall x\in O_e:\ \alpha_x\restrict
   F_{\kappa_e}(I_e,\Gamma_e,\omega_e),
  \]
  where \(F_\kappa\) is the registered per-edge upper bound. This local
  condition does not replace the global predicate.

  \item[I5 --- Relational global envelope.]
  For every authority domain \(\rho\),
  \begin{equation}
    \mathsf{Safe}_\rho(s,\Project_\rho(s)).
    \label{eq:global-invariant}
  \end{equation}

  \item[I6 --- Linearized effect binding.]
  For every effect receipt
  \(r\in\mathit{EffectReceipts}_{s_{k+1}}\setminus
  \mathit{EffectReceipts}_{s_k}\) that first appears at position \(k+1\) of an
  accepted trace, there is exactly one transition
  \(s_k\xrightarrow{\mathsf{CommitEffect}(p_u)}s_{k+1}\). In the pre-state
  \(s_k\), its exact bound row \(a\), instance \(x\), episode \(u\), map
  commitment \(m\), handle \(\chi\), authority domain
  \(\rho=\rho(\mathsf{root}(a))\), complete epochs, and domain revision all
  match current state;
  \(\mathit{Handles}_{s_k}(\chi)=a\),
  \(A_{s_k}(\mathsf{slot}(a))=\mcode{active}\),
  \(\mathsf{CurrentRow}_{s_k}(a,x)\), \(\mathsf{Exec}_{s_k}(x)\),
  \(p_u\in\mathit{EffectPermits}_{s_k}\),
  \(\mathsf{now}(s_k)<\mathsf{expiry}(p_u)\), and
  \(\mathsf{effectSlot}(p_u)\notin\mathit{UsedEffectKeys}_{s_k}\);
  \(u\in\sem{\alpha_x}\),
  \(\mathsf{EpisodeCurrent}_{s_k}(u)\),
  \(\mathsf{Obl}_{s_k}(o_{\alpha_x},u)\), and
  \(\mathsf{Safe}_\rho(s_k,\Project_\rho(s_k))\). The accepted decision,
  protected-effect result, receipt, and previously unused slot enter
  \(s_{k+1}\) atomically and are thereafter immutable.
\end{description}

Let \(\mathsf{Base}_\Gamma\) contain exactly the authenticated source and
capability seed nodes directly authorized by \(\Gamma\); in the base profile,
it contains no active capability whose provenance root differs from
\(\Gamma\). Define
\(\mathsf{Reach}_\Gamma\) as the least set satisfying
\[
\begin{array}{ll}
 \textsc{Base}:&\mathsf{Base}_\Gamma\subseteq\mathsf{Reach}_\Gamma,\\[2pt]
 \textsc{Step}:&\displaystyle
 \frac{\begin{gathered}
   \kappa\in K_\Gamma\quad \omega\in\Omega_\kappa\\
   I\in\mathcal I_\kappa\quad
   \mathsf{components}(I)\subseteq\mathsf{Reach}_\Gamma\\
   \{x\}\in\mathcal O_\kappa\quad R_\kappa(I,\{x\})\\
   \alpha_x\restrict F_\kappa(I,\Gamma,\omega)
 \end{gathered}}
 {x\in\mathsf{Reach}_\Gamma},
\end{array}
\]
where \(\Omega_\kappa\) is the registered output-ordinal set.  The registered
semantic closure is the inductively generated set
\[
 \mathsf{Cl}_\Gamma=
 \bigcup\{\sem{\alpha_x}\mid
 x\in\mathsf{Reach}_\Gamma\text{ is a capability instance}\}.
\]
Write \(\mathsf{Active}_\Gamma(s)\) for instances \(x\) such that
\(\mathsf{Exec}_s(x)\) and a current derivation rooted at \(\Gamma\) occurs
in \(\Project_\Gamma(s)\).  This inductive definition requires no closure of
the descriptor language under union.

\subsection{Necessity and separation results}

\begin{proposition}[Post-fulfillment observation is necessary]
\label{prop:observation}
Fix an envelope \(\Gamma\) and a normalized pre-activation graph \(H\). Let a
monitor decide activation using only a transaction tuple
\(z=\langle \mathit{caller},\mathit{endpoint},\mathit{amount},
\mathit{scope},\mathit{mandate},\mathit{declaredItem},\mathit{receipt}\rangle\).
If two admissible provider states have the same \(z\) but resolve to capability
descriptors \(\alpha_1\) and \(\alpha_2\), let \(e_1\) and \(e_2\) be the
corresponding candidate singleton-output edges. If the envelope admits exactly
one of \(H\boxplus_{\rho(\Gamma)}(\{e_1\}\cup\mathsf{out}(e_1))\) and
\(H\boxplus_{\rho(\Gamma)}(\{e_2\}\cup\mathsf{out}(e_2))\), no decision function of \(z\)
alone is both sound and complete for activation.
\end{proposition}

This result does not identify a defect in any transaction protocol. It shows
that activation needs either authenticated post-fulfillment observation or a
provider profile that supplies a sound upper bound on the actual output.

\begin{proposition}[Componentwise checks do not preserve correlation]
\label{prop:correlation}
There exists a downward-closed acquisition envelope for which every field of
two output descriptors belongs to its corresponding component allowlist, but
activating both violates \(\Phi_\Gamma\).
\end{proposition}

The read-or-publish co-possession rule is a witness. Therefore the main safety
statement cannot be reduced to a union of allowed field values.

\begin{proposition}[Monetary compliance is not acquisition safety]
\label{prop:zero-cost}
For any positive monetary bound, there exist an envelope \(\Gamma\) and a
zero-price acquisition trace
that preserves the monetary invariant but violates a descendant-cardinality
or co-possession clause of \(\Phi_\Gamma\).
\end{proposition}

\subsection{Safety theorems}

\begin{theorem}[Quarantine non-authority]
\label{thm:quarantine}
Under A1, A2, A3, A6, and A10, creation, payment, receipt, quarantine, and
resolution neither create an active record nor make an output executable.
\(\mathsf{CommitActivation}\) is the only transition that creates an active
record, and only a subsequent \(\mathsf{PublishHandle}\) for that durable
record can make the output executable.
\end{theorem}

\begin{theorem}[Backed activation]
\label{thm:backed}
Under A1--A11, every executable capability in a reachable
state has a current derivation to exactly one live root envelope, while I5
accounts for all live roots in that root's authority domain.
\end{theorem}

\begin{theorem}[Acquisition non-amplification]
\label{thm:nonamp}
Under A1--A11, Equation~\eqref{eq:global-invariant} holds in
every reachable accepted state, including domains containing several grant
identifiers. Consequently, for every live root \(\Gamma\),
\[
 \bigcup_{x\in \mathsf{Active}_\Gamma(s)}\sem{\alpha_x}
 \subseteq \mathsf{Cl}_\Gamma,
\]
while the relational theorem additionally preserves cardinality, identity,
co-possession, and graph clauses not expressible by this episode-set
corollary.
\end{theorem}

\begin{theorem}[Split non-evasion]
\label{thm:split}
Under A1, A4--A6, A8, and A11, fix a reachable refresh-closed state \(s\),
and let two acquisition plans differ only in order
partition, retry or grant identifier, surface alias, independent-event order,
or allocation among descendants whose roots share authority domain \(\rho\).
Let \(H=\Project_\rho(s)\). If normalization yields the same prospective
addition \(\Delta\), then
\(\mathsf{Safe}_\rho(s,H\boxplus_\rho\Delta)\) has the same value for both
plans. While the live-grant set is fixed and no exogenous transition removes
authority from the projection, no interleaving can make all of an unsafe
\(\Delta\) active. A safe fragment may commit and advance the domain revision;
each later fragment must prepare or revalidate against that new revision and
current projection, and the first fragment that would cross the envelope is
rejected. An accepted time, epoch, grant-currentness, fence, or destruction
transition atomically refreshes the domain; evaluation then restarts from its
new live-grant set, revision, and projection without exposing an unsafe
intermediate state.
\end{theorem}

\begin{theorem}[Crash-safe at-most-once activation]
\label{thm:crash}
Under A1--A11, for every fault-extended accepted trace containing
crashes, finite recovery prefixes, duplicate messages, delayed receipts, and
retries, every state \(s\) occurring in that trace, and every instance \(x\),
\[
\begin{aligned}
 \mathsf{Exec}_s(x)\Rightarrow
 \exists!a\in\mathit{Active}_s:\quad
 &\mathsf{CurrentRow}_s(a,x)\land\mathsf{SlotConsumed}_s(a)\\
 &{}\land A_s(\mathsf{slot}(a))=\mcode{active}
 \land\mathsf{Safe}_{\rho_a}(s,\Project_{\rho_a}(s)).
\end{aligned}
\]
Moreover, for every such state \(s\) and every activation slot \(j\),
\[
 \left|\{a\in\mathit{Active}_s:\mathsf{slot}(a)=j\}\right|\leq 1.
\]
Thus one logical activation slot cannot produce two active records or two
simultaneously executable authorities.
\end{theorem}

\begin{theorem}[Refund non-resurrection]
\label{thm:refund}
Under A1, A2, A4--A6, A8, and A9, if
\(\mathsf{RefundSettled}(o)\) takes \(s\) to \(s'\), then
\(C_{s'}(o)=\mcode{refunded}\) while
\(\Project_\rho(s')=\Project_\rho(s)\) for every authority domain \(\rho\).
Thus a refund, discount, free credit, or restored source reservation cannot by
itself release a semantic capacity clause. Capacity changes require a separate
authorized transition---either fencing (followed, if desired, by ordinary
activation under a narrower descriptor) or verified destruction---under its
registered semantics.
\end{theorem}

\begin{theorem}[Epoch non-inheritance]
\label{thm:epoch}
Under A1, A2, A6, A8, and A10, advancing any bound grant, policy, identity,
contract, resource, or revocation epoch prevents an old activation permit,
effect permit, or handle from authorizing a new protected effect. Changed
resources require resolution and activation against current state.
\end{theorem}

\begin{theorem}[Resource-to-effect confinement]
\label{thm:confinement}
Under A1--A11, let \(s_0\rightarrow\cdots\rightarrow s_n\) be an accepted
trace and let \(r\) first enter \(\mathit{EffectReceipts}\) in
\(s_k\rightarrow s_{k+1}\). If \(r\) and \(p_u\) witness
\(\mathsf{AcceptedEffect}_{s_{k+1}}(x,u)\), then that transition is the unique
\(\mathsf{CommitEffect}(p_u)\) that accepts \(u\). In its pre-state \(s_k\),
\(u\in\sem{\alpha_x}\), the exact row for \(x\) has a complete current
acquisition derivation, and the joint relational envelope for its authority
domain holds. The episode's registered state predicate and logical-time
interval hold there, and its continuing obligations are satisfied. The effect
result and \(r\) are atomically bound to that same single-use permit.
\end{theorem}

\section{Provenance-Bounded Activation Design}
\label{sec:design}

The design uses narrow adapters around a common broker. Runtime-specific code
constructs an acquisition effect; provider-specific code quarantines and
resolves an output; the authorization kernel evaluates registered mathematical
objects; and a separate effect broker controls use. No model-generated text is
an authorization input.

\subsection{Acquisition-effect IR}

Table~\ref{tab:ir} lists the required intermediate-representation fields. The
IR is typed and closed: an unknown resource kind, provider profile, or field
name is rejected before dispatch.

\begin{table*}[t]
\caption{Security-relevant acquisition-effect IR.}
\label{tab:ir}
\small
\begin{tabularx}{\textwidth}{L{0.19\textwidth} L{0.29\textwidth} Y}
\toprule
Field group & Bound values & Security purpose \\
\midrule
Origin & root grant, task/session, canonical controller, runtime profile, request ID & attributes the proposal to one current authorization root \\
Operation & acquisition kind, provider, product/profile version, exact parameters, logical and derived root-bound provider keys & prevents operation, cross-root, and retry substitution \\
Inputs & source references, reservation evidence, contributing capability nodes & identifies every consumed or supporting input \\
Intended output & resource kinds, maximum output descriptors, multiplicity, intended beneficiary & authorizes only a bounded prospective conversion \\
Correlation & purpose, target, registered target--domain map, data domain, audience, delegation, descendants, co-possession group & preserves whole clauses instead of fieldwise unions \\
Freshness & grant, policy, contract, identity, provider-resource, and revocation epochs; expiry & makes state changes invalidate prepared authority \\
Evidence & required receipt type, provider state query, projector digest, unresolved-field policy & specifies how actual output is established \\
\bottomrule
\end{tabularx}
\end{table*}

The adapter must inject a stable acquisition request identifier when the
underlying runtime supplies no stable one. A retry reuses the same logical key;
the broker deterministically derives a provider-facing key bound to the grant
root. Changing a transport or order ID cannot allocate a new semantic activation
slot, and two roots that reuse the same raw key cannot alias one provider
operation.

\subsection{Three orthogonal state axes}

A single status flag conflates economically and security-distinct outcomes.
The authorization layer persists the commerce, delivery, and authority coordinates of
the formal root-qualified output-state map \(\sigma_s(o,\omega)\) defined in
Section~\ref{sec:formal}. Thus a payment can be settled while activation is denied; a resource can be
quarantined while settlement is indeterminate; and money can be refunded while
the capability remains active. One provider operation has one commerce
coordinate but separate delivery coordinates and activation slots for every
registered output ordinal, so a multi-output result can contain active,
quarantined, and fenced outputs simultaneously. Transition guards constrain valid combinations
without forcing them into one linear enumeration.

\subsection{Proposal and quarantine}

Algorithm~\ref{alg:acquire} authorizes the acquisition call but deliberately
stops before authority creation. \(\mathsf{Current}\) validates the live root,
every bound epoch, contract membership, authorized source and provider
profiles, purpose membership, the registered evidence profile, and a current
same-root derivation for every capability input;
\(\mathsf{Reserve}\) delegates source accounting to a registered adapter; and
\(\mathsf{VaultReceive}\) returns only broker-internal resource identifiers.
\(\mathsf{PersistOutboxIfCurrent}\) runs as a serializable transaction: it
re-reads the immutable operation and permit ledgers, the source reservation,
the live root and epochs, and the logical clock; it persists the exact outbox
row only if every dispatch guard still holds.

\begin{algorithm}[t]
\caption{Acquire into quarantine}
\label{alg:acquire}
\begin{algorithmic}[1]
\Require root envelope \(\Gamma\), normalized request \(q_a\), state \(s\)
\State \(\kappa\gets\mathsf{ResolveContract}(q_a.contract,q_a.version)\)
\If{\(\kappa=\bot\) or \(\neg\mathsf{Current}(\Gamma,\kappa,q_a,s)\)}
  \State \Return \(\mathsf{deny}(\mcode{E\_STALE\_OR\_UNKNOWN})\)
\EndIf
\State \(\mathit{cr},b\gets\mathsf{CanonicalIdentities}(q_a,s)\)
\If{ambiguous \(\mathit{cr}\) or \(b\)}
  \State \Return \(\mathsf{deny}(\mcode{E\_IDENTITY\_AMBIGUOUS})\)
\EndIf
\State \(o\gets\langle\Gamma.\mathit{gid},q_a.\mathit{opKey}\rangle\)
\State \(\widehat{o}\gets\mathsf{CanonicalOperationRecord}
  (\Gamma,q_a,\kappa,\mathit{cr},b)\)
\If{\(\mathit{Ops}_s(o)\notin\{\bot,\widehat{o}\}\)}
  \State \Return \(\mathsf{deny}(\mcode{E\_OPERATION\_KEY\_CONFLICT})\)
\EndIf
\State \(p_a\gets\mathsf{InstallOrMatchProposal}(o,\widehat{o},\Gamma,q_a,s)\)
\State \(r_s\gets\mathsf{ReserveSource}(p_a,q_a.inputs,o)\)
\State \(w\gets\mathsf{PersistOutboxIfCurrent}(p_a,r_s,o)\)
\If{\(w=\bot\)}
  \State \Return \(\mathsf{deny}(\mcode{E\_DISPATCH\_NOT\_CURRENT})\)
\EndIf
\State \(y\gets\mathsf{ProviderDispatch}(w,\mathsf{providerKey}(o))\)
\If{\(y\in\{\mcode{timeout},\mcode{pending},\mcode{unknown}\}\)}
  \State \Return \(\mathsf{indeterminate}(\mcode{E\_PROVIDER\_UNKNOWN})\)
\EndIf
\If{\(\neg\mathsf{AuthenticatedCorrelation}(y,w)\)}
  \State \Return \(\mathsf{deny}(\mcode{E\_RECEIPT\_MISMATCH})\)
\EndIf
\State \(X\gets\mathsf{VaultReceive}(y,p_a,o)\)
\State \Return \(\mathsf{quarantined}(X)\)
\end{algorithmic}
\end{algorithm}

\FloatBarrier

Provider timeout does not imply absence. When terminal state is not
authoritatively known, the operation remains
\(C_s(o)=\mcode{indeterminate}\);
any observed late output is correlated into quarantine.

\subsection{Resolution and atomic activation}

The resolver queries authoritative current state rather than trusting the
requested product description. It normalizes default identities, network
reachability, roles, audiences, data domains, exportability, delegation,
descendant behavior, concurrency, validity, and revocation handles according
to the registered profile. It records an explicit \(unresolvedFields\) set.
An unresolved required field either maps to the profile's conservative upper
bound or returns \(\bot\).

Algorithms~\ref{alg:activate} and~\ref{alg:activate-commit} separate the two
checks. Preparation constructs a prospective graph and signs its exact root.
Commit repeats every current-state predicate inside the durable transaction.
This prevents a time-of-check to time-of-use change between resolution and
activation.

\begin{algorithm}[t]
\caption{Resolve and prepare one quarantined output}
\label{alg:activate}
\small
\begin{algorithmic}[1]
\Require quarantined output \(x\) at ordinal \(\omega\), root-qualified key \(o\), state \(s\)
\State \(\widehat{o}\gets\mathit{Ops}_s(o)\)
\If{\(\widehat{o}=\bot\)}
  \State \Return \(\mathsf{deny}(\mcode{E\_UNKNOWN\_OPERATION})\)
\EndIf
\State \((\Gamma,\kappa,\pi,\mathit{opKey})\gets
  (\Gamma_o,\kappa_o,\pi_o,\mathsf{rawKey}(o))\)
\State \(\mathit{rr}\gets\mathsf{ResolveActualCapability}(x,o,\omega,\pi,\kappa,s)\)
\If{\(\mathit{rr}=\mcode{correlation\mbox{-}mismatch}\)}
  \State \Return \(\mathsf{deny}(\mcode{E\_RECEIPT\_MISMATCH})\)
\EndIf
\If{\(\mathit{rr}=\bot\)}
  \State \Return \(\mathsf{indeterminate}(\mcode{E\_UNRESOLVED\_OUTPUT})\)
\EndIf
\State \(((\alpha_x,h_x,U_x),s_r)\gets\mathit{rr}\)
\State \(e\gets\mathsf{CandidateEdge}(o,x,\alpha_x,h_x,U_x,\omega,s_r)\)
\If{\(\Gamma\notin\mathsf{Live}_{\rho(\Gamma)}(s_r)\) or
      \(\kappa\notin K_{\Gamma}\) or \(\omega\notin\Omega_{\kappa}\) or
      \(I_e\notin\mathcal I_{\kappa}\) or \(\{x\}\notin\mathcal O_{\kappa}\) or
      \(\neg\mathsf{CurrentInputsAtRoot}(I_e,\Gamma,s_r)\) or
      \(\neg\mathsf{ActivationFresh}_{s_r}(x,o,\omega)\) or
      \(\neg R_{\kappa}(I_e,\{x\})\) or
      \(\mathsf{now}(s_r)\notin\upsilon_{\alpha_x}\) or
      \(\mathsf{Cyclic}(G_{s_r}\cup\{e,x\})\)}
  \State \Return \(\mathsf{deny}(\mcode{E\_CONVERSION})\)
\EndIf
\If{\(\alpha_x\not\restrict F_{\kappa}(I_e,\Gamma,\omega)\)}
  \State \Return \(\mathsf{deny}(\mcode{E\_OUTPUT\_BOUND})\)
\EndIf
\State \(\rho\gets\rho(\Gamma)\); \(H'\gets\Project^{+}_{\rho}(s_r,e)\)
\If{\(\neg\mathsf{Safe}_{\rho}(s_r,H')\)}
  \State \Return \(\mathsf{deny}(\mcode{E\_ENVELOPE})\)
\EndIf
\State \(p_x\gets\mathsf{PrepareExactActivation}
  (x,e,H',\mathit{opKey},\omega,s_r)\)
\State \Return \(\mathsf{prepared}(p_x)\)
\end{algorithmic}
\end{algorithm}

\begin{algorithm}[t]
\caption{Commit activation and publish its opaque handle}
\label{alg:activate-commit}
\small
\begin{algorithmic}[1]
\Require prepared permit \(p_x\), quarantined output \(x\), root-qualified key \(o\)
\State \((\Gamma,\rho,\kappa,\pi,\omega)\gets
  \mathsf{BoundActivationContext}(p_x)\)
\If{\(\rho\neq\rho(\Gamma)\)}
  \State \Return \(\mathsf{abort}(\mcode{E\_ACTIVATION\_RACE})\)
\EndIf
\State \(s'\gets\mathsf{BeginSerializableTransaction}()\)
\State \(\widehat{o}'\gets\mathit{Ops}_{s'}(o)\)
\If{\(\widehat{o}'=\bot\) or
      \(\neg\mathsf{OperationMatchesPermit}(\widehat{o}',p_x,o)\)}
  \State \Return \(\mathsf{abort}(\mcode{E\_ACTIVATION\_RACE})\)
\EndIf
\State \(r'\gets\mathsf{RequeryAndResolve}(x,\pi,\kappa,s')\)
\If{\(r'=\bot\)}
  \State \Return \(\mathsf{abort}(\mcode{E\_ACTIVATION\_RACE})\)
\EndIf
\State \((\alpha'_x,h'_x,U'_x)\gets r'\)
\State \(e'\gets\mathsf{ReconstructCandidateEdge}
  (o,x,\alpha'_x,h'_x,U'_x,\omega,s',\rho)\)
\State \(H''\gets\Project^+_\rho(s',e')\)
\If{\(p_x\notin\mathit{ActivationPermits}_{s'}\) or
      \(\neg\mathsf{ExactPermitBinding}(p_x,e',H'',\alpha'_x,h'_x,U'_x,s')\) or
      \(\mathsf{now}(s')\geq\mathsf{expiry}(p_x)\) or
      \(\Gamma\notin\mathsf{Live}_{\rho}(s')\) or
      \(\kappa\notin K_{\Gamma}\) or \(\omega\notin\Omega_{\kappa}\) or
      \(\omega_{e'}\neq\omega\) or
      \(I_{e'}\notin\mathcal I_{\kappa}\) or \(\{x\}\notin\mathcal O_{\kappa}\) or
      \(\neg\mathsf{CurrentInputsAtRoot}(I_{e'},\Gamma,s')\) or
      \(\neg\mathsf{ActivationFresh}_{s'}(x,o,\omega)\) or
      \(\neg R_{\kappa}(I_{e'},\{x\})\) or
      \(\alpha'_x\not\restrict F_{\kappa}(I_{e'},\Gamma,\omega)\) or
      \(\mathsf{now}(s')\notin\upsilon_{\alpha'_x}\) or
      \(\mathsf{Cyclic}(G_{s'}\cup\{e',x\})\) or
      \(\neg\mathsf{Safe}_{\rho}(s',H'')\)}
  \State \Return \(\mathsf{abort}(\mcode{E\_ACTIVATION\_RACE})\)
\EndIf
\If{\(\mathsf{slot}(p_x)\in UsedActivationKeys_{s'}\)}
  \State \Return \(\mathsf{deny}(\mcode{E\_ACTIVATION\_REPLAY})\)
\EndIf
\State \(a\gets\mathsf{CommitActivationAndConsumeSlot}(p_x,e',s')\)
\State \(h\gets\mathsf{PublishHandleIfCurrent}(a)\)
\If{\(h=\bot\)}
  \State \Return \(\mathsf{deny}(\mcode{E\_HANDLE\_NOT\_CURRENT\_OR\_COLLISION})\)
\EndIf
\State \Return \(\mathsf{active}(h)\)
\end{algorithmic}
\end{algorithm}

\FloatBarrier

A successful projector result may retain a nonempty \(U_x\) only when every
listed field has the registered sound upper bound of
Equation~\eqref{eq:resolver-upper-bound}; the bound, the unresolved-field set,
and the evidence are committed by \(h_x\). A field with no such bound makes
the projector return \(\bot\), so the algorithm never treats omission as a
wildcard.

The active handle is deterministically derived from the durable record.
\(\mathsf{PublishHandleIfCurrent}\) implements the formal
\(\mathsf{PublishHandle}\) transition as a serializable transaction that
re-resolves the row, verifies the exact current row and active slot, checks that
the handle is absent or already maps to that row, and atomically installs the
mapping. A concurrent state change or collision therefore returns \(\bot\)
without exposing a handle. A crash after activation commit and before
publication repeats the same transaction and, when its guards still pass,
republishes the same handle rather than creating another authority.

\subsection{Relational aggregate evaluation}

The kernel evaluates \(\mathsf{Safe}_\rho\) over the union of canonical active
nodes from every live grant in the authority domain plus the candidate edge.
A profile may implement each conjunct as a decision diagram, policy program,
bounded query, or compiled set of clauses. The registered implementation must
be total, deterministic, and authority-domain-extensional. Its commitment
includes the policy source, compiler, normalization profile, and test vectors.

For example, let one clause permit at most one task-local worker, and another
permit either \(\mcode{sensitive.read}\) or
\(\mcode{external.publish}\) for a controller root but forbid their
co-possession. A sequence of free enrollments is evaluated against the count
after alias collapse. Separately allowed read and publish outputs are evaluated
against their joint controller projection. Field membership alone never
authorizes the combination.

For multi-output acquisitions, each ordinal has its own activation slot and is
admitted against the current aggregate. Safe prefixes may activate while the
first violating output remains quarantined; this is the operational meaning
of split non-evasion.

\subsection{Operation-time use}

Algorithm~\ref{alg:use} validates the complete current path, not only the
opaque handle. The broker resolves the actual effect into an episode and
evaluates the exact current lineage before issuing a single-use effect permit.
The permit is not an instruction to dispatch. It is a closed commitment to the
episode, opaque handle, durable active record, authority domain, complete epoch
vector, and root revision. The effect executor validates that immutable permit
and consumes its single-use slot at its own linearization point. Consequently,
a lineage or policy change
between preparation and execution cannot inherit the earlier decision.

\begin{algorithm}[t]
\caption{Prepare and atomically commit a protected use}
\label{alg:use}
\small
\begin{algorithmic}[1]
\Require opaque handle \(\chi\), proposed external effect \(\widetilde{u}\), current state \(s\)
\State \((x,a,\Gamma)\gets\mathsf{ResolveActiveLineage}(\chi,s)\); \(\rho\gets\rho(\Gamma)\)
\If{any object is absent, \(\mathit{Handles}_s(\chi)\ne a\),
      \(A_s(\mathsf{slot}(a))\ne\mcode{active}\),
      \(\neg\mathsf{CurrentRow}_s(a,x)\), \(\neg\mathsf{Exec}_s(x)\),
      fenced, revoked, or digest-mismatched}
  \State \Return \(\mathsf{deny}(\mcode{E\_STALE\_LINEAGE})\)
\EndIf
\State \(u\gets\mathsf{NormalizeProtectedEffect}(\widetilde{u},s)\)
\If{\(u=\bot\)}
  \State \Return \(\mathsf{deny}(\mcode{TARGET\_DOMAIN\_MAPPING\_INVALID})\)
\EndIf
\If{\(u\notin\sem{\alpha_x}\) or
      \(\neg\mathsf{EpisodeCurrent}_s(u)\) or
      \(\neg\mathsf{Obl}_s(o_{\alpha_x},u)\) or
      \(\neg\mathsf{Safe}_\rho(s,\Project_\rho(s))\)}
  \State \Return \(\mathsf{deny}(\mcode{E\_EFFECT\_OUTSIDE\_CLOSURE})\)
\EndIf
\State \(\beta\gets\mathsf{EffectBinding}(u,\chi,a,\rho,
  \mathsf{targetDomainMap}(s),\mathsf{epochs}(s),
  \mathsf{rootRevision}(\rho,s),\widetilde{u}.\mathsf{opKey})\)
\State \(p_u\gets\mathsf{PrepareSingleUseEffectPermit}(\beta)\)
\State \(s'\gets\mathsf{BeginSerializableEffectTransaction}()\)
\State \((x',a',\Gamma')\gets\mathsf{ResolveActiveLineage}(\chi,s')\)
\State \(u'\gets\mathsf{NormalizeProtectedEffect}(\widetilde{u},s')\)
\If{\(u'=\bot\), or any resolved object is absent, fenced, revoked, or
      digest-mismatched, or \(\mathit{Handles}_{s'}(\chi)\ne a'\), or
      \(A_{s'}(\mathsf{slot}(a'))\ne\mcode{active}\), or
      \(\neg\mathsf{CurrentRow}_{s'}(a',x')\), or
      \(\neg\mathsf{Exec}_{s'}(x')\)}
  \State \Return \(\mathsf{abort}(\mcode{E\_EFFECT\_COMMIT\_RACE})\)
\EndIf
\State \(\rho'\gets\rho(\Gamma')\)
\State \(m'\gets\mathsf{targetDomainMap}(s')\); \(\eta'\gets\mathsf{epochs}(s')\)
\State \(r'\gets\mathsf{rootRevision}(\rho',s')\)
\State \(\beta'\gets\mathsf{EffectBinding}
  (u',\chi,a',\rho',m',\eta',r',\widetilde{u}.\mathsf{opKey})\)
\If{\(p_u\notin\mathit{EffectPermits}_{s'}\), or
      \(\mathsf{now}(s')\geq\mathsf{expiry}(p_u)\),
      or \(u'\notin\sem{\alpha_{x'}}\), or
      \(\neg\mathsf{EpisodeCurrent}_{s'}(u')\), or
      \(\neg\mathsf{Obl}_{s'}(o_{\alpha_{x'}},u')\), or
      \(\neg\mathsf{ExactEffectBinding}(p_u,\beta')\), or
      \(\neg\mathsf{Safe}_{\rho'}(s',\Project_{\rho'}(s'))\)}
  \State \Return \(\mathsf{abort}(\mcode{E\_EFFECT\_COMMIT\_RACE})\)
\EndIf
\If{\(\mathsf{effectSlot}(p_u)\in UsedEffectKeys_{s'}\)}
  \State \Return \(\mathsf{deny}(\mcode{E\_EFFECT\_REPLAY})\)
\EndIf
\State \Return \(\mathsf{CommitEffectAndConsumeSlot}(p_u,u',s')\)
\end{algorithmic}
\end{algorithm}

\FloatBarrier

The exact effect binding contains the normalized episode---episode identifier,
subject, beneficiary, effect, target, data domain, purpose, recipient,
registered state precondition, and half-open validity interval---and separately
binds the target-to-domain map identifier, version, and digest. The finite
registered map is total only over its
supported targets; an unknown target, wrong domain, or stale map commitment
fails before permit issuance and is checked again at commit. The
effect slot is scoped by the authority domain
\(\rho=(\mathit{issuer},\mathit{cr},\mathit{policyEpoch})\) and a logical idempotency
key. In particular, neither a grant identifier nor a retry identifier defines
a fresh authority domain.
The registered finite reference profile maps a target into its canonical data
domain and a recipient into its authorized audience, fixes the reference
issuer explicitly, and records only an inert effect-commit receipt. It performs
no external action. A deployment must place
permit verification, one-time effect-slot consumption, and the irreversible
effect in the same trusted broker transaction, or have the recipient atomically redeem the
permit before applying an idempotent effect. Merely checking in one process and
calling an unbound executor afterwards does not implement this algorithm.

\subsection{Refunds, retirement, and late results}

A refund changes commerce state, not delivery or authority state. The active
node remains in \(\Project_{\rho(\Gamma)}(s)\) until authoritative destruction or a
registered permanent receiver fence proves that it cannot be exercised.
Agent assertion, deletion of a local alias, loss of a UI reference, or source
budget restoration is not destruction evidence.

Late provider success is handled symmetrically. Reconciliation identifies the
logical operation and canonical resource, records commerce and delivery state,
and places every output in quarantine. It never consumes a prior activation
slot or treats a prepared permit as implicit acceptance; the output must pass ordinary current-state resolution,
preparation, and commit.
After any restart, recovery first re-queries provider and broker state,
re-runs the registered projector, and reconstructs the current graph or
episode binding. It may proceed only through the ordinary commit guard after
exact comparison with the durable prepared permit; it never jumps directly
from an old permit to activation or effect acceptance.

Table~\ref{tab:crash} maps each externally visible crash window to the durable
recovery rule required by the formal fault semantics.

\begin{table}[H]
\caption{Crash windows and durable recovery outcomes.}
\label{tab:crash}
\small
\begin{tabularx}{\columnwidth}{L{0.33\columnwidth} Y}
\toprule
Crash point & Recovery rule \\
\midrule
Before provider dispatch & retry the same durable outbox/key; no output authority exists \\
After create, before receipt & query the same logical key; any output enters quarantine \\
After resolution, before activation & discard stale preparation and resolve current state again \\
After activation commit, before handle publication & republish the deterministic handle for the same active record \\
After handle publication & durable active record already precedes the handle \\
After effect-permit preparation & re-resolve current lineage and episode; abort stale or changed binding before acceptance \\
After effect commit, before receipt return & return the same durable receipt; the consumed effect slot forbids a second acceptance \\
After refund & preserve delivery and authority state until verified retirement \\
\bottomrule
\end{tabularx}
\end{table}

\subsection{Stable denial taxonomy}

Auditability requires the checker to verify where a trace failed. The core
families are: malformed or unknown IR, stale grant/profile/identity/resource
epoch, ambiguous identity, source-evidence failure, receipt mismatch,
unresolved actual output, manifest mismatch, invalid conversion, cycle or
lineage break, relational-envelope violation, activation replay, commit-order
violation, refund-as-destruction, stale handle, and effect outside closure.
For the two-stage use path, distinct codes identify missing permits, exact
binding mismatch, commit-time lineage or epoch races, and consumed effect-slot
replay.
Each concrete fixture registers one expected phase and error code; a mutant
that happens to fail earlier for an unrelated reason does not count as
detection of its target invariant.

\FloatBarrier

\section{Protocol and Runtime Instantiations}
\label{sec:instantiations}

The proposed architecture is an authorization layer, not a replacement for commerce,
identity, agent-communication, or tool protocols. An adapter assigns each
external field one of four evidence classes:

\begin{enumerate}
  \item \emph{native authenticated}: defined and authenticated by the upstream
  protocol;
  \item \emph{authenticated sidecar}: supplied by a separately registered
  signer and explicitly outside the upstream protocol;
  \item \emph{independently observed}: reconstructed from an authoritative
  provider read such as resource inspection; or
  \item \emph{unavailable}: cannot authorize activation when policy requires
  the field.
\end{enumerate}

Natural-language descriptions never promote an unavailable field to an
authenticated one.

\subsection{Agentic Payment Protocol}

The AP2 v0.2 mapping admits a credential only after a separate upstream
verifier has validated its signature/disclosure proof, credential type,
issuer/key binding, and applicable trust policy. Under that explicit premise,
the registered profile validates the frozen schema branches, timestamps,
Checkout Mandate digest, Payment Mandate transaction identifier, receipt
references, confirmation binding, and cross-record correlation
\cite{ap22026spec}. These records produce \mcode{transactionEvidence}; they do
not produce an activation certificate or an actual-capability manifest.

For experiments that need provider-return semantics, we define a separate
experimental delivery attestation (version 1). It binds the transaction,
Checkout Mandate and line item, provider and product profile, projector digest,
canonical resource and beneficiary, actual-capability-manifest digest,
delegability, validity, revocation handle, and evidence class. This object is
an experimental activation-layer record. It is not represented as part of AP2 or
as a protocol extension adopted by AP2.

This mapping respects the transaction layer's own scope. A successful AP2
check is required where applicable and remains logically independent of the
subsequent resource-resolution and activation decisions.

\subsection{UCP, A2A, and MCP}

UCP uses ``capability'' for merchant-supported commerce features
\cite{ucp2026}; this paper uses \emph{authorization capability} for a set of
exercisable protected episodes. The UCP adapter retains the upstream meaning
and maps order and fulfillment records only into commerce and delivery
evidence.

An A2A Agent Card supplies declared remote-agent identity, supported Skills,
interfaces, and authentication metadata \cite{a2a2026spec}. The declaration
supports discovery and provider-profile selection. Actual beneficiary,
delegation, descendant, revocation, and effect fields still require the
registered authenticated or independently observed evidence classes. An A2A
task therefore does not become an active acquired principal merely because its
Agent Card is syntactically valid.

MCP provides the structured tool-call channel and OAuth-based authorization
context \cite{mcp2026authorization,mcp2026tools}. The acquisition adapter binds
the runtime identity, task root, server identity, tool name, and complete
arguments at the final broker boundary. A tool result that identifies a new
resource is directed to quarantine rather than returned as a usable raw
credential. Invoking a resulting handle is separately mediated through the
single-use effect-permit commit. Ordinary tools that do not create an external
account, credential, tenant, provider resource, or principal are outside the
acquisition profile.

Table~\ref{tab:mapping} summarizes which evidence each protocol layer supplies,
which authority facts it deliberately does not infer, and where the activation layer
uses the resulting record.

\begin{table*}[t]
\caption{Evidence responsibility across protocol layers.}
\label{tab:mapping}
\small
\begin{tabularx}{\textwidth}{L{0.14\textwidth} L{0.25\textwidth} L{0.25\textwidth} Y}
\toprule
Layer & Native evidence used & Evidence not inferred & Activation-layer use \\
\midrule
AP2 & mandate, payment, receipt, transaction binding & actual resource authority, canonical controller, later activation & validate transaction-side prerequisites \\
UCP & commerce discovery, checkout/order and fulfillment records & authorization meaning of a delivered resource & identify declared product and lifecycle evidence \\
A2A & Agent Card, endpoint, declared Skills and authentication metadata & actual derived principal and complete downstream effect set & select and authenticate a remote-agent evidence profile \\
MCP & structured request, server/tool identity, OAuth context & semantics of a newly created provider resource & capture the acquisition request at the runtime boundary \\
Provider inspection & actual resource configuration in the registered observable profile & undocumented or hidden provider behavior & derive the actual-capability manifest \\
Provenance-bounded authorization & envelope, graph, resolver, lineage, epochs, activation and effect permits & payment validity or provider facts not supplied by an adapter & make the activation decision and linearize each protected effect \\
\bottomrule
\end{tabularx}
\end{table*}

Table~\ref{tab:deployment-mapping} makes the resulting deployment contract
explicit.  Each row is a conformance template: the named interception,
resolution, quarantine, activation, and effect-commit points must all be
implemented for that profile.  The table does not claim that the four example
profiles are evaluated production integrations.

\begin{table*}[t]
\caption{Deployment mapping from acquired resources to post-fulfillment authorization
enforcement points.}
\label{tab:deployment-mapping}
\scriptsize
\setlength{\tabcolsep}{2.5pt}
\renewcommand{\arraystretch}{1.08}
\begin{tabularx}{\textwidth}{L{0.09\textwidth} L{0.12\textwidth} L{0.14\textwidth} L{0.12\textwidth} L{0.16\textwidth} L{0.14\textwidth} Y}
\toprule
Acquired resource & Interception point & Authoritative resolver & Quarantine primitive & Activation predicate & Effect linearization & Fail-closed behavior \\
\midrule
MCP credential & final structured tool-call/result boundary at the acquisition broker & issuer/provider introspection plus canonical identity and registered scope profile & secret retained in the broker vault; no raw credential returned & resolved scope, audience, exportability, delegation, epochs, and aggregate envelope pass & broker or recipient atomically redeems the exact single-use permit with API acceptance & missing issuer, scope, audience, or mediated redemption leaves the output quarantined; a raw-return path is nonconforming \\
Cloud instance & brokered create and correlated provider result & provider inspection of canonical resource, IAM, network, storage, and runtime state & create disabled or non-routable, or withhold access behind a gateway & actual manifest and aggregate graph pass; commit-time requery shows no drift & start, attach, expose, or a protected workload effect atomically accepts the permit & ambiguous identity or state yields indeterminate; absence of quarantine or atomic acceptance makes the profile nonconforming \\
A2A child agent & child registration/task result before endpoint or credential publication & registered identity/delegation service and provider-side child state; Agent Card only selects the profile & withhold endpoint and credentials; child remains non-dispatchable & canonical parent and beneficiary, descendant, delegation, effect, epoch, and aggregate bounds pass & brokered task dispatch atomically redeems the permit at the child gateway & card-only or incomplete downstream fields, ambiguous aliases, or a non-redeemable endpoint prevent dispatch \\
AP2/UCP service & verified transaction/fulfillment result at the broker & provider account/entitlement read plus a registered delivery attestation, when used & fulfilled account, token, or entitlement retained in the broker vault & transaction prerequisites and resolved entitlement, beneficiary, delegability, epochs, and aggregate envelope pass & service invocation or recipient gateway atomically redeems the exact permit & payment or receipt alone, unavailable resolution, or ambiguous fulfillment leaves the output quarantined \\
\bottomrule
\end{tabularx}
\end{table*}

The delivery attestation in the AP2/UCP row is the experimental activation-layer
record defined above, not a native AP2 or UCP assertion.  Across all rows,
``atomic'' has the A10 meaning: permit validation and one-time slot consumption are
bound to effect acceptance at the registered linearization point.

\subsection{Agent-runtime mappings}

The evaluated profile uses two separately implemented adapters. One targets
the final structured MCP client-call boundary in the Rust-based Codex CLI
0.154.0; the other targets the corresponding boundary in the TypeScript-based
Gemini CLI 0.59.0 \cite{openai2026codexcli,google2026geminicli}. They share a
declarative acquisition-effect profile, stable enums, and golden vectors, but
not mapping functions. The source audit freezes the exact upstream commit and
boundary locator for each target; neither adapter is represented as an
upstream integration patch.

Each adapter maps a registered capture representation of the final structured
request at its audited MCP boundary. The record includes session/task root,
turn or task identity, tool-use identity, server identity, tool name, and
complete arguments. The mapping requires a stable
\mcode{acquisitionRequestId}; a conforming integration must inject and persist
one when the runtime supplies none. The registered profile excludes generic
shell execution, direct network access, raw container sockets, and
unregistered MCP servers from its complete-mediation claim.

The evaluated mapping establishes field completeness, request attribution,
and fail-closed handling of missing values. A conforming runtime integration
enforces that mapping at the audited interception point. The acquisition
broker remains the security boundary: interception alone cannot authorize a
returned capability, and restating an intended constraint cannot bypass the
broker.

The executable upstream layer uses the exact frozen client implementations
against byte-identical deterministic stdio MCP witnesses.  For Gemini, the path
includes Core tool discovery, \mcode{DiscoveredMCPTool} execution, response
transformation, and the final \mcode{Client.callTool} dispatch.  For Codex, it
includes initialization, discovery, and the lower-level
\mcode{RmcpClient::call\_tool} dispatch that the registered higher-level binding
invokes.  The Codex \mcode{PreparedMcpCall} capture point is source-audited but
not directly executed by this harness.  Task root, session, turn, tool-use, and
request identifiers are harness-supplied sidecar context rather than MCP wire
fields; a digest binds that complete context to the service-observed tool and
arguments before either adapter can map the call.

\subsection{Container-resource and composed runtime profile}

A local container profile gives a concrete, non-commerce resource lifecycle.
The broker alone controls the container interface. It creates a deterministically
named container without starting it, inspects the returned canonical container
ID and configuration, resolves an actual manifest, issues and validates an
implementation-specific container configuration/start certificate derived from
the committed activation record, re-inspects the container, and only then
starts that same container. This certificate is a Docker-adapter refinement
object, not a second definition of the formal activation permit \(p_x\).

The safe profile requires an image pinned by digest, no network, read-only root
filesystem, all Linux capabilities dropped, no new privileges, no host mount,
device, published port, privileged mode, GPU, or automatic restart. CPU and
memory bounds, entrypoint, environment-reference policy, namespace mode, and
all security options are part of the actual manifest. Configuration drift
between inspection and start invalidates the derived start certificate and
therefore prevents effect-permit redemption and start dispatch.

The container engine's create, inspect, and start operations supply independent
evidence for quarantine, resolution, certificate-gated start, and drift
rejection. This local profile does not establish durable activation-record
atomicity, arbitrary cloud policy, or containment of a compromised container
engine. Its purpose is to exercise a real resource handle without network
access, credentials, or production effects.

A separate staged composition selects one of three independently observed
upstream calls for every registered logical case, verifies its wire/context
binding, and admits it to the corresponding adapter, common IR, and reference
state machine. Each case advances only until its registered rejection boundary
or authorized completion.  For the 12 cases that reach resource creation, the
broker creates the digest-pinned container in a non-running state; Docker
inspection supplies the independently observed isolation and resource fields,
while the registered compute-product profile and runtime binding supply the
semantic subject, beneficiary, purpose, and effect fields.  The four authorized
first starts occur only after resolution, authority-domain activation, opaque
handle publication, effect-permit preparation, and re-inspection of the same
container.  The start authorizer then commits the exact effect permit
immediately before dispatching the Docker start command.

This experiment tests an executable evidence path across previously separate
components.  Its deterministic stimulus does not invoke a model, remote MCP
transport, provider, or complete CLI session, and no upstream production source
file is patched.  The reference record and Docker Engine remain separate
durability domains; the established result is broker-side pre-dispatch
mediation, not engine-internal effect finality.

\subsection{Provider-specific cloud mappings}

AWS role passing and Google Cloud service-account attachment demonstrate that
resource creation and attached identity are distinct authorization concerns
\cite{aws2026passrole,google2026actas}. The proposed architecture does not replace these
native controls. A conforming adapter records the native decision as evidence,
then projects the resulting role, audience, data domains, delegation, network,
and resource identity into the common manifest. The cross-provider theorem
applies to the registered common fields; provider-native behavior outside the
profile remains governed by that provider.

\section{Registered Evaluation}
\label{sec:evaluation}

The evaluation answers separation, safety, utility, portability, and cost
questions without executing payments, creating online accounts, launching
cloud resources, or exposing usable credentials. The default suite uses inert
capability descriptors. A separate local container suite exercises
create--inspect--certificate--re-inspect--start ordering under a
network-disabled profile and is reported under its narrower integration
boundary.

\subsection{Questions and hypotheses}

\begin{description}[leftmargin=3em,style=nextline]
  \item[EQ1: Strict separation.]
  Do the six registered observation-limited separation predicates accept a
  fixture that satisfies its declared observed fields but returns a semantic
  capability outside the acquisition envelope?

  \item[EQ2: Safety.]
  Does the full authorization profile reject every registered unbacked, widened, stale,
  split, substituted, co-possessed, multiply activated, or improperly
  committed effect?

  \item[EQ3: Utility.]
  Does it accept supported benign acquisitions in every registered resource
  class without a false rejection?

  \item[EQ4: Portability.]
  Can two separately implemented runtime adapters map their final broker
  requests to the same IR obligations, and can public protocol schemas supply
  their declared evidence classes?

  \item[EQ5: Cost.]
  What local latency and evidence size are attributable to normalization,
  resolution, relational graph evaluation, activation, effect-permit
  preparation, and effect commit?
\end{description}

The registered primary safety criterion is zero unauthorized activation and
zero protected-effect escape in the finite registered denominator. Every zero
is reported with its exact denominator; the corresponding universal claims
come from the conditional proofs, not from extrapolating the finite runs.

\subsection{Resource fixtures}

The main matrix contains five resource classes: funds-backed service, API
quota, compute environment, credential, and child-agent identity. Each class
contains eight unsafe transforms and four benign controls, for 40 unsafe and
20 benign base fixtures. The unsafe families are:

\begin{enumerate}
  \item stronger actual output than the declared product;
  \item alias or order splitting;
  \item cross-agent aggregation or prohibited co-possession;
  \item zero-price or restored-credit acquisition;
  \item retry or duplicate-output multiplication;
  \item activation before quarantine or resolution completes;
  \item an epoch or authority-domain revision change after effect-permit
  preparation but before effect commit; and
  \item drift in an effect permit's bound subject, target, recipient, purpose,
  data domain, or effect at commit.
\end{enumerate}

Five additional replay probes, one per resource class, attempt to consume an
already committed effect slot.  They are reported separately from the
60-fixture primary matrix so they do not alter its preregistered class and
family denominator.

The 40-trace refinement set exercises obligations not represented by the
one-output base fixtures. It comprises one accepted multi-output acquisition;
two output-ordinal failures; two activation-currentness failures; five
publication-currentness cases, of which three drifts reject and two safe
changes accept; six acquisition-permit expiry, epoch, identity, retry, and
operation-key cases; four submitted-handle binding, retry, and collision cases;
five effect-currentness cases; six provider-receipt or root-qualified
provider-key cases; two output-association cases; three explicit
provider/source/purpose envelope-authorization failures; one missing
live-grant-membership case; one missing consumed-activation-slot case; and two
target-to-data-domain failures. The target mapping is a finite, versioned
registry whose identity and digest are bound into each effect permit; the
registered target is deliberately not the same string as its mapped data
domain.

A separate fresh-grant regression instantiates two envelopes that differ only
in \mcode{grantId} and two service-entitlement capabilities under a
one-active-by-kind ceiling. Each capability is admissible in isolation. The
test requires both the reference evaluator and the independently implemented
checker to derive the same authority domain and reject their joint projection;
it is also outside the 60-fixture primary denominator.

The two runtime adapters replay all 60 registered logical mapping cases,
producing 120 adapter traces paired by logical case. The local container suite
contains exactly ten benign and ten unsafe configurations.

The upstream client-component experiment registers nine logical cases for each
runtime: one benign case and eight cases that subsequently exercise an
unregistered tool, missing sidecar context, capture-binding drift, a missing
effect permit, commit-time effect drift, an epoch change, post-certificate
Docker configuration drift, or replay after one authorized start. Each frozen
client executes every case in three fresh client/server process pairs, producing
54 service-observed stdio calls. All 54 calls are transport observations; the
unsafe label identifies the later boundary mutation or gate attempt, not a
rejection by the upstream client.

The staged composition selects repetition one for each logical case and admits
the resulting 18 observed calls to capture validation and the local path. Each
case advances until its registered rejection boundary or authorized completion;
12 reach Docker quarantine and four issue an authorized first start. The runner
records requested and completed Docker start operations; an unsafe case passes
only if its unauthorized attempt adds no start-command request. The payment
profile uses structurally
valid, inert synthetic mandate/receipt records checked against four schemas
frozen from AP2 v0.2; these records are not upstream examples or live payment
credentials. Targeted mutations cover credential type, issuer, digest,
transaction and receipt correlation, confirmation, validity, and success/error
branches. Fault evidence covers 15 activation-path durable crash cuts, five
effect-commit crash cuts, and the complete 32-value schedule space of a
five-bit replay/interleaving model for one logical key. It separately exercises
32 effect-permit replays, 192 receipt substitutions, seven stale-activation
recovery probes, five post-commit/pre-publication drift probes, eight
contract-security probes, ten effect-security probes, and 16 late-success
schedules.

\subsection{Mandatory separating witnesses}

Three witnesses are release gates:

\begin{itemize}
  \item a correctly paid and delivered low-cost API product whose actual
  credential is redelegable or administrative rather than task-local read;
  \item several zero-price acquisitions by aliases or descendants that would
  exceed one canonical active-descendant clause; and
  \item a refund that restores source budget without destroying the first
  credential, followed by a second acquisition under a one-active-credential
  envelope.
\end{itemize}

Additional witnesses combine two separately admissible capabilities into a
forbidden read-and-publish pair, substitute a stronger provider version, and
deliver two outputs after timeout and retry. Every witness uses inert local
handles.

\subsection{Observation-limited separation predicates}

The separation experiment instantiates six registered observation-limited
predicates: amount-only, OAuth-scope-only, mandate-only,
fixed-resource-budget, admission-only, and a request-field-complete pre-action
predicate. Each receives only the proposal-side, transaction-side, or
admission-side fields named by its definition. These proposition-derived test
instruments are neither implementations nor emulations of AP2, OAuth, a
production budget monitor, an admission-control system, or any cited work;
their sole role is to exercise the observation boundary in
Proposition~\ref{prop:observation}. The full profile receives the same
proposal plus the quarantined actual output and current activation graph.

A fixture registers \mcode{applicableBaselines}. Before inclusion, each
separation predicate must pass a \mcode{baselineSelfCheck} proving that the
fixture satisfies the predicate's own validity conditions. An inapplicable
predicate is reported as N/A and excluded from its denominator. The frozen
separation corpus deliberately supplies the declared inputs for all six
predicates in every fixture, so each has the same 60-case denominator. Strict
separation means that such a predicate accepts according to its own contract
while the actual output violates the acquisition envelope. This measured
acceptance is attributable only to the registered predicate--fixture pair; it
is not attributed to any external protocol or implementation.

\subsection{Independent checker and mutants}

The reference implementation emits canonical event traces. A separately
implemented checker imports neither the reference state transition function,
normalizer, envelope evaluator, any separation-predicate implementation, nor
the reference decision predicate. It
independently reconstructs the event digest chain, three state axes, canonical
identities and authority domains, operation keys, actual-manifest and operation
bindings, complete-clause fit, the active-domain aggregate envelope, epochs,
single-use activation and effect slots, commit-before-publish ordering, refund
semantics, exact effect-permit binding, root revision, and the resulting effect
receipt. The formal hypergraph, rooted-derivation, and conversion-edge results
remain proof obligations under the assumptions in
Table~\ref{tab:assumptions}.

The executable semantic matrix instantiates eight unsafe families for each
resource class: stronger actual output, alias splitting, cross-agent
co-possession, restored credit after a zero-price acquisition, duplicate retry,
pre-resolution activation, stale epoch, and use-time drift. Separately, 89
checker negative tests are registered. They mutate every top-level field of
the acquisition permit (23), activation permit (25), effect permit (16), and
effect receipt (10), each while recomputing the event digest chain. The
remaining 15 cover hash-bound and semantic payloads, required activation and
effect-receipt ledger rows, the outbox provider key, logical-time discipline,
source-only predecessor typing, source and destruction evidence, orthogonal
state, commit-before-publish ordering, operation-time effect, refund
occupancy, the registered outcome, and canonical-identity acyclicity. A
semantic case scores only when its registered rejection phase and stable error
code match; a checker negative test passes only on independent rejection.
These tests do not stand in for the arbitrary-hypergraph proof obligations
stated above.

The relationship verifier additionally checks all 40 refinement traces in the
independent checker, 118 acquisition-permit bindings, 81 append-only
activation-permit bindings, 103 resolved rows, 109 output associations, 71
committed-row handle bindings, all 30 canonical effect episodes, and all 117
source reservations across the base and refinement traces. It also checks one
authenticated destruction receipt per resource class, five
same-output-ordinal replay rejections, all five registered target-to-domain
bindings, both fail-closed target-map probes, and four explicit logical-time
advance events.

\subsection{External-source closure}

The external structure audit contains 32 frozen public source units from AP2,
UCP, A2A, MCP, and RepliBench \cite{ap22026spec,ucp2026,a2a2026spec,
mcp2026tools,black2025replibench}. Together they cover all five registered
resource classes. For every unit, the audit records an exact immutable locator,
the upstream object and fields, the paper's safe resource-class abstraction,
and one of four evidence classes for each of the 39 activation fields. The
audit uses public structures and safe abstract labels; it neither copies
deployable examples nor treats a structural mapping as an operational incident.
This evidence establishes source and field closure, not attack prevalence.

\subsection{Metrics and reporting}

Primary metrics are unauthorized capability activation rate, active tuples
without a current root, split/alias bypass, duplicate active authority,
protected-effect escape, stale revival, and phase-accurate rejection. Utility
metrics are benign completion and false rejection by resource class. Cost
metrics are p50/p95/p99 local latency for each kernel phase and serialized
evidence size, including effect-permit preparation and effect commit. Container
execution and runtime-adapter mapping are reported separately from policy-kernel
time.

The result record binds each numerator and denominator to fixture and profile
digests, runtime/protocol versions, exhaustive domains or fixed execution
orders, host hardware where timing is reported, and stable error codes.
Every reported result is sourced from a named frozen record that passes the
independent relationship checks described in
Section~\ref{sec:evidence-status}; the release audit freezes those records and
the corresponding manuscript together.

\section{Evaluation Results}
\label{sec:evidence-status}

The semantic, checker, adapter, and protocol-profile experiments execute
offline with inert identifiers and no external provider effects. A fresh
upstream-component build may fetch lockfile-bound dependencies under its
recorded install policy; the resulting MCP trials use only local stdio and no
provider. A semantic effect commit records only an inert local receipt at the
registered linearization point. Table~\ref{tab:evaluation-summary}
reports exact registered denominators. The container row is a separate local
integration result; it is not folded into the provider-free semantic
denominator.

\begin{table*}[t]
\caption{Registered safety, composition, portability, and external-evidence results, including the staged case-to-boundary matrix.}
\label{tab:evaluation-summary}
\footnotesize
\renewcommand{\arraystretch}{0.96}
\begin{tabularx}{\textwidth}{L{0.20\textwidth} L{0.29\textwidth} Y}
\toprule
Evidence layer & Denominator & Observation \\
\midrule
Reference semantics and trace checker & 60 base fixtures; 810 canonical events; 40 refinement traces; 5 effect-replay probes; 89 checker tamper tests; 1 fresh-\mcode{gid} aggregate regression & 20/20 benign accepted and committed an inert effect; 40/40 unsafe rejected at the registered phase and code; checker agreement 60/60 and refinement agreement 40/40; 89/89 tamper tests and 5/5 replays rejected; the joint fresh-\mcode{gid} aggregate rejected by both evaluators \\
Observation-limited separation predicates & 6 registered predicates; 60 constructed fixtures each & each accepted 60/60, including 40/40 envelope-violating witnesses; proposition-level separation test, not a full-system comparison \\
Runtime-boundary mapping & 60 paired logical cases; 120 adapter replays; 48 fail-closed probes & 60/60 pairs produced identical complete IR obligations; 48/48 probes rejected; 2/2 frozen source boundaries passed \\
Frozen upstream MCP client components & 2 source trees; 9 logical cases per runtime; 3 fresh client/server process pairs per case & 54/54 stdio calls completed initialize, discovery, invocation, server observation, and response; all normalized wire/context bindings agreed across repetitions \\
AP2 target profile & 4 primary and 5 dependency schemas; 1 positive; 25 mutants in 13 families; 39 activation fields & 9/9 source files matched frozen paths, bytes, and digests; 1/1 positive was structurally consistent but remained non-active; 25/25 mutants rejected; 39/39 fields classified \\
Crash and replay & 15 activation cuts; 5 effect cuts; 32 schedules; 32 effect replays; 192 receipt substitutions; 7 activation-recovery probes; 5 publication drifts; 8 contract-security probes; 10 effect-security probes; 16 late-success schedules & no duplicate active record, published handle, or accepted effect; every registered cut, schedule, replay, substitution, recovery, security, publication-drift, and late-arrival outcome passed \\
Local container integration & 10 benign and 10 unsafe configurations & 10/10 benign completed; 10/10 unsafe attempts produced no unauthorized start; all 21 created containers were removed \\
Upstream-observed-to-container composition & 2 benign and 16 unsafe paths; 2 runtime adapters; 12 real containers & 2/2 benign paths completed; all 16/16 unsafe attempts issued zero unauthorized Docker start-command requests; all 12 containers were removed \\
Public-source closure & 32 units, 5 sources, 5 resource classes, 39 fields; 1,248 classifications & every pair classified; 0/32 source units and 0/5 per-class coverage unions supplied a complete activation profile by themselves \\
\bottomrule
\end{tabularx}

\vspace{0.25em}
\textit{Staged composition by logical case.} Each row contains one Codex and
one Gemini execution; ``created'' and ``start'' are exact observed operation
counts, not inferred states.

\scriptsize
\setlength{\tabcolsep}{3pt}
\begin{tabularx}{\textwidth}{L{0.18\textwidth} L{0.34\textwidth} L{0.15\textwidth} Y}
\toprule
Logical case & Earliest terminal boundary and stable result & Container created & Authorized start / later attempt \\
\midrule
Complete chain & broker completion, exit status 0 & 2/2 & 2/2 first starts \\
Unregistered tool & capture/IR gate: \mcode{TOOL\_NOT\_REGISTERED} & 0/2 & 0 starts \\
Missing runtime context & capture gate: \mcode{UPSTREAM\_CONTEXT\_MISSING} & 0/2 & 0 starts \\
Capture-binding drift & capture gate: \mcode{UPSTREAM\_CAPTURE\_BINDING\_MISMATCH} & 0/2 & 0 starts \\
Missing effect permit & start gate: \mcode{E\_EFFECT\_PERMIT\_REQUIRED} & 2/2 & 0 starts \\
Effect-binding drift & effect commit: \mcode{EFFECT\_PERMIT\_BINDING\_MISMATCH} & 2/2 & 0 starts \\
Epoch drift & effect commit: \mcode{EFFECT\_COMMIT\_RACE} & 2/2 & 0 starts \\
Container-configuration drift & re-inspection: \mcode{E\_CONFIGURATION\_DRIFT\_RESOURCES} & 2/2 & 0 starts \\
Single-use replay & second certificate redemption: \mcode{E\_CERTIFICATE\_REPLAY} & 2/2 & 2/2 first starts; 0/2 second start requests \\
\bottomrule
\end{tabularx}
\end{table*}

\subsection{Semantic separation and independent reconstruction}

The reference execution and checker agreed on every fixture. The 40 unsafe
traces divide evenly across eight transformation families and five resource
classes. Twenty-five reject during activation, five during delivery, and ten
during effect commit. Six stable codes account for every rejection: 15
\nolinkurl{RELATIONAL_ENVELOPE_VIOLATION} outcomes and five each for
\nolinkurl{OUTPUT_ASSOCIATION_CONFLICT},
\nolinkurl{OUTPUT_OUTSIDE_RELATIONAL_CLAUSE},
\nolinkurl{INVALID_AUTHORITY_TRANSITION}, \nolinkurl{EFFECT_COMMIT_RACE}, and
\nolinkurl{EFFECT_PERMIT_BINDING_MISMATCH}.

The checker rejected 89/89 additional trace-tampering tests. Field-complete
subsets mutate the 23 acquisition-permit, 25 activation-permit, 16
effect-permit, and ten effect-receipt top-level fields independently while
retaining the corresponding identifier and rehashing the trace. The remaining
15 cover hash-bound and semantic payloads, required ledger rows, the outbox
provider key, logical-time discipline, source-only predecessor typing, source
and destruction evidence, orthogonal state, commit ordering, operation-time
effect, refund occupancy, registered outcome, and canonical-identity
acyclicity. Thus the four field-complete subsets rejected 23/23, 25/25, 16/16,
and 10/10, and the remaining suite rejected 15/15.

All 20 benign traces prepared one exact single-use effect permit, consumed its
slot, and produced one inert effect receipt. Five epoch/revision races and five
commit-time binding drifts rejected before the effect linearization point;
five additional replay probes rejected reuse of a consumed effect slot.  The
authority-domain regression changes only the fresh grant identifier and
confirms that both envelopes map to the same
$\langle\mathit{issuer},\mathit{canonicalRoot},\mathit{policyEpoch}\rangle$
domain and the same effect-idempotency slot.  Under a one-active-by-kind
ceiling, each of two capabilities is separately admissible, while their joint
projection is rejected by both the reference evaluator and the independently
implemented checker. Thus a fresh grant identifier is provenance, not a means
to reset domain accounting.

All 40 refinement traces also agreed under independent reconstruction. Their
registered outcomes include the multi-output, ordinal, activation-currentness,
publication-currentness, acquisition-permit, submitted-handle,
effect-currentness, provider-correlation, output-association,
envelope-authorization, live-membership, consumed-slot, and target-map cases
specified in Section~\ref{sec:evaluation}. Of the five publication-currentness
cases, expiry, epoch, and manifest drift rejected without publishing a handle,
while an unrelated state update and safe graph growth were accepted.

Across the base and refinement evidence, the verifier checked 118 exact
acquisition-permit bindings and 114 durable dispatch-outbox records under the
source-only input profile; 81 append-only activation-permit bindings; 103
resolved rows; 109 root-qualified-operation/output-ordinal associations; and
71 submitted opaque handles bound to committed active rows. It also checked 30
canonical effect episodes, including explicit state preconditions, time
intervals, expiry, and all six current epochs; 117 authenticated source
reservations; five authenticated destruction receipts; five
same-output-ordinal replay rejections; and four explicit logical-time advances.
The provider-key tests accepted one distinct-root operation and rejected one
cross-root receipt; exact output-association retry returned the original
resource, while a conflicting association rejected. The handle tests accepted
one exact retry, rejected two ownership collisions and one wrong submitted
handle, and returned the original committed handle on retry. The finite
ordinal registry bound each output to its contract, version, and resource
class. The versioned target-domain registry exercised one non-identity mapping
in each of the five resource classes, and its identity, version, and digest
remained bound to every one of the 30 effect permits.

Every registered observation-limited separation predicate accepted all 60
constructed fixtures because each fixture intentionally satisfies the fields
that predicate observes. In
particular, the 40 separating witnesses are not protocol-invalid requests;
their unsafe fact is in the actual returned capability or prospective graph.
This result establishes the observation boundary of
Proposition~\ref{prop:observation}. It is neither a system-level comparison nor
a failure-rate estimate for complete AP2, OAuth, budget, or admission-control
implementations.

\subsection{Protocol and public-source closure}

The dedicated AP2 profile freezes four primary and five referenced v0.2 JSON
schemas at one immutable upstream commit. Its positive record set is synthetic,
inert, and structurally valid; credential signature, disclosure, issuer-key,
and trust-policy verification are explicit upstream premises. The checker then
evaluates cross-record hashes, transaction and receipt references, issuer and
time consistency, success/error branches, confirmation binding, and
product/resource correlation. It also classifies all 39 activation fields as
upstream, authenticated after that premise, provider-supplied, or locally
derived. AP2 supplies transaction evidence in this mapping; it supplies neither
the actual-capability manifest nor an activation certificate.
All nine frozen schema files matched their registered paths, byte counts, and
digests. The synthetic positive passed its structure and correlation profile
and remained non-active because independent resource evidence was
intentionally absent; all 25 targeted mutants across 13 families rejected
with their registered codes. The
field mapping contains two upstream, six authenticated, 13 provider, and 18
derived fields.

The broader source audit freezes seven units each from AP2, UCP, A2A, and MCP,
and four from RepliBench. Across the resulting 1,248 unit--field pairs, 11 are
native authenticated, six require an authenticated sidecar, one is
independently observable, and 1,230 are unavailable from that source unit.
Every unavailable entry remains explicit rather than being inferred from
unauthenticated description. Consequently, the 0/32 complete-unit result is a
fail-closed closure finding: the public structures contribute evidence, while
the resource resolver, current graph, epochs, and activation transaction supply
the distinct fields needed to authorize activation.

\subsection{Runtime, crash, and resource evidence}

The two adapter implementations share only a declarative profile and golden
vectors. They target the frozen final structured MCP call paths identified in
Section~\ref{sec:instantiations}; direct network, generic shell, raw container
socket, and unregistered-server paths are outside those profiles and fail
closed. The 48 negative replays cover unregistered boundaries, servers,
transports, and tools; missing or unknown fields; runtime/profile mismatch;
invalid arguments; and correlation mismatch.

The frozen upstream experiment executed nine cases three times through each
client component.  All 54 fresh client processes negotiated with 54 fresh,
separate stdio witness processes, discovered the advertised tool, sent the full
arguments, received the deterministic response, and matched the server-side
observation.  The independent verifier reconstructed every argument, sidecar,
capture, wire, and response digest from the registries and confirmed distinct
process identities.  The Gemini path reached Core discovery,
\mcode{DiscoveredMCPTool}, response transformation, and
\mcode{Client.callTool}.  The Codex path reached
\mcode{RmcpClient::call\_tool}; its higher-level \mcode{PreparedMcpCall} boundary
was source-audited but was not directly executed.

At the Codex release tag, \mcode{Cargo.lock} labels 150 workspace-local packages
as version \mcode{0.0.0} while the release workspace declares
\mcode{0.154.0}.  The offline build therefore used a temporary detached clone
whose only lock normalization changed those 150 local version labels to
\mcode{0.154.0}; external dependency resolution was unchanged.  The original
and derived lock digests, compiler/toolchain, harness, built binary, and both
Codex source boundaries are recorded, and the frozen source checkout remained
clean.  The Gemini runner created a detached temporary clone, installed from the
committed lockfile under a recorded registry-access policy, and built Core
without changing the supplied checkout. Across all 27 fresh client processes,
one stable loaded-module path set covered 446 local distribution files and
1,965 files from 276 lockfile-owned external packages. The record binds every
loaded file digest, package version, lockfile integrity value, portable
installation path, Node/npm version, and aggregate closure digest. The
independent verifier rebuilt a second detached clone and rehashed the complete
2,411-file executed closure.

The crash harness reconstructs a new broker instance solely from checksummed
canonical JSON stored with file synchronization, atomic rename, and directory
synchronization. Across all 15 activation cuts and 32 schedules, each logical
operation created at most one active record and one published handle. Across
all five effect-commit cuts, the used slot, accepted decision, inert result,
linearization witness, and receipt were co-committed; the effect permit
remained append-only.
All 32 effect-permit replay attempts and all 192 registered receipt
substitutions failed closed, including the 32 output-association conflicts.
Seven activation-recovery probes rejected stale permits after restart. Five
post-commit/pre-publication expiry, epoch, state, manifest, or root drifts also
rejected after restart without publishing a handle. All eight contract-security
and all ten effect-security probes rejected their registered mutations, and all
16 delayed successes returned to quarantine before explicit reconciliation.
This profile covers broker-process restarts and an inert idempotent provider
ledger; it does not model hardware loss or faults inside the file system's
atomic-write primitive.

The Docker suite uses one already-present digest-pinned image, performs no
pull, and applies a non-root, network-disabled, read-only, capability-dropped,
no-new-privileges profile with no mounts, devices, published ports, or restart
policy. Ten distinct CPU/memory/PID configurations completed the required
create--inspect--certificate--re-inspect--start sequence. Ten missing,
malformed, tampered, wrong-scope, replayed, expired, or configuration-drift
cases were rejected with zero successful-start increment. This result
establishes certificate-gated local start and drift detection, not durable
activation-record atomicity or protection for starts outside the broker.

The staged composition selected the first server-observed call for each of the
18 logical cases and admitted it to a staged process beginning with capture
validation and IR normalization. Each case advanced only until its registered
rejection boundary or authorized completion. Both benign paths exited with
status zero, 12 cases reached real Docker quarantine, and four authorized first
starts reached effect-permit commit and broker-mediated dispatch. All 16 unsafe
attempts issued zero unauthorized Docker start-command requests.
For each runtime, an unregistered tool, missing sidecar context, and altered
capture binding rejected before resource creation; missing permit, effect
drift, stale epoch, and post-certificate configuration drift rejected before
start dispatch; and certificate replay rejected without a second request after
one authorized first use.  All 12 created containers were removed.  The four
authorized starts and effect-permit commits comprise the two benign paths and
the first use in each replay case.

These observations establish execution of frozen upstream MCP client
components and a digest-bound staged handoff into the local broker path.  They
do not represent a complete interactive CLI session, model-driven tool
selection, remote MCP or network-provider execution, durable cross-system
atomicity, Docker Engine internal finality, or production deployment.

\subsection{Cost and repeatability}

The policy-kernel benchmark executes 20 warm-up and 200 measured rounds over
the fixed 60-fixture order, yielding 12,000 complete traces on an Apple M4 Max
with Node.js 24.18.0. Table~\ref{tab:latency} reports every measured in-memory
phase in the frozen kernel record. Adapter execution is reported by agreement
and protocol outcomes, not folded into these timings. Durable I/O, network,
provider execution, container creation, and production concurrency are
excluded.

\begin{table}[t]
\caption{Local in-process latency in microseconds; sample counts reflect which
actions occur in the fixed 60-fixture matrix over 200 measured rounds.}
\label{tab:latency}
\small
\begin{tabularx}{\columnwidth}{Y r r r}
\toprule
Path (samples) & p50 & p95 & p99 \\
\midrule
Grant installation (12,000) & 38.250 & 51.208 & 55.125 \\
Proposal normalization (16,000) & 226.584 & 264.458 & 360.209 \\
Source reservation (16,000) & 24.875 & 27.750 & 32.834 \\
Dispatch (16,000) & 227.250 & 264.541 & 357.084 \\
Delivery lifecycle (17,000) & 17.083 & 19.083 & 22.917 \\
Quarantine receive (17,000) & 48.833 & 54.417 & 62.708 \\
Lifecycle update (3,000) & 20.459 & 24.375 & 29.458 \\
Actual-capability resolution (15,000) & 52.334 & 58.334 & 67.083 \\
Relational graph gate (16,000) & 277.625 & 321.542 & 423.792 \\
Activation commit (11,000) & 253.375 & 295.583 & 403.375 \\
Handle publication (11,000) & 191.292 & 229.542 & 342.333 \\
Effect-permit preparation (6,000) & 267.667 & 296.084 & 419.792 \\
Effect commit (6,000) & 288.958 & 320.000 & 441.542 \\
Complete semantic trace (12,000) & 2041.625 & 3418.709 & 3601.875 \\
\bottomrule
\end{tabularx}
\end{table}

The complete semantic event ledger contains 810 JSON Lines records totaling
1,105,230 bytes: mean 1,364.5, p50 872, p95 3,547, p99 3,667, and maximum
3,731 bytes per event. The 60 serialized traces total 2,691,339 bytes: mean
44,855.7, p50 44,860, p95 74,292, p99 77,508, and maximum 77,508 bytes per
trace. Semantic, crash/replay, external-source, and AP2 runners each compare
repeated deterministic payloads byte-for-byte. Timing fields are reported as
measurements rather than deterministic artifacts.

\FloatBarrier

Formal results in Section~\ref{sec:formal} remain conditional mathematical
statements proved in Appendix~\ref{app:proofs}. The finite observations above
validate the registered implementations and source profiles; they neither
replace the proofs nor enlarge the complete-mediation and evidence assumptions
in Section~\ref{sec:assurance}.

\section{Related Work}
\label{sec:related}

\subsection{Protection, usage control, and stateful authorization}

Least privilege and complete mediation are foundational protection principles
\cite{saltzer1975protection}. HRU formalized rights acquisition and established
the undecidability boundary for general safety \cite{harrison1976protection}.
State-transition comparisons of access-control lists, trust management, and
capability systems make acquisition and delegation explicit as changes to
authorization state \cite{chander2001state}. Delegation Logic and distributed
proving cast cross-domain authorization as proof of compliance assembled from
policies and credentials \cite{li2003delegationlogic,bauer2005distributedproving};
proof-carrying authentication likewise separates proof construction from
simple request-side checking \cite{appel1999pca}. Our certificate adopts that
checkable-evidence discipline for a different transition: after fulfillment,
the concrete provider object, its aliases, controller, capabilities, and
aggregate edges may enter the active authority graph.

Our architecture deliberately uses a finite registered conversion closure and a
decidable envelope predicate. UCON models mutable attributes and ongoing
authorization \cite{park2004ucon}; ABAC and zero-trust architectures provide
attribute and resource-centric policy foundations
\cite{nist800162,nist800207}. Stateful least privilege permits rules that
depend on prior cloud actions \cite{cao2024stateful}. Our distinct object is
the graph transition that admits a provider-created resource as authority,
including actual-output resolution, alias aggregation, and refund-safe
retirement.

Relationship-based authorization systems demonstrate that tuple-defined
authorization graphs can be evaluated at production scale while respecting
causal ordering across access-control and object changes
\cite{pang2019zanzibar}. Such a system decides requests over already admitted
relations; the proposed admission predicate instead controls whether a newly
delivered resource and its relations may join the active graph.

Macaroons and OAuth mechanisms support contextual caveats, audience
restriction, structured authorization, and token exchange
\cite{birgisson2014macaroons,rfc8707,rfc8693,rfc9396}. The authorization layer can encode
their authenticated constraints inside a capability descriptor or evidence
profile. Formal web-model analysis also shows that OAuth's guarantees depend on
explicit protocol and threat assumptions \cite{fett2016oauth}. Even when those
guarantees hold, token validity alone does not establish that an externally
returned resource is admissible under the full acquisition graph.

\subsection{Agent authorization and dynamic capabilities}

ToolEmu exposes consequential failures across high-stakes agent toolkits, while
AgentDojo evaluates tool-using agents when returned data may carry adversarial
instructions \cite{ruan2024toolrisks,debenedetti2024agentdojo}. These
environments establish the practical need to mediate agent--tool boundaries;
they do not define the authorization-state transition by which a
provider-created account, credential, entitlement, or principal becomes active.

The OWASP Agent Control Standard (ACS) v0.1 standardizes an Observed
Agent--Guardian Agent JSON-RPC contract with lifecycle hooks, including tool-call
request/result and subagent start/stop, five decision dispositions, and trace,
inspection, provenance, and cryptographic-signing profiles together with the
ACS-Core audit chain
\cite{owasp2026acs}. It provides a portable intervention and observability
plane while leaving policy semantics to the Guardian. The proposed activation policy can run
as such a policy: hooks capture acquisition dispatch and delivery, while
the policy enforces quarantine, authenticated actual-state resolution,
aggregate graph admission, atomic activation, and effect-permit redemption.
This mapping is our instantiation, not an ACS guarantee; ACS in turn
standardizes hook, response, and observability surfaces outside
the activation policy. Thus ACS transports the intervention while our policy
defines its acquisition-specific decision and state transition.

Task-scoped authorization and agent-identity protocols bind actions to task or
delegation context \cite{sharma2026pauth,prakash2026aip}. Bounded Agents adds
principal chains, accumulated state, composition closure, and delegation
bounds \cite{muruaga2026bounded}. HCP studies execution invariants for
MCP-style runtimes, including canonical resources and handle provenance
\cite{liu2026executioncontrol}. These mechanisms govern invocations within an
existing authority model. Our architecture constructs the provenance and
activation decision for a newly delivered account, credential, entitlement,
environment, or principal before that object joins the model.

Pre-action policy enforcement and deployed layered access-control systems bind
agent identity, request parameters, execution context, and signed decisions
before a tool operation \cite{uchibeke2026oap,malik2026granular}. Governance
frameworks separately distinguish technical capability from allowed autonomy
\cite{zheng2026autonomylevels}. The proposed architecture shares their separation of
capability and permission while moving the decision object to the actual
provider-created resource and its aggregate acquisition graph after
fulfillment.

EBL-Core binds one fully materialized candidate action, policy versions, typed
evidence, a verifiable decision derivation, and current-state redemption
through an execution-release contract \cite{wu2026eblcore}. Our work
addresses a different temporal object: the provider-created output is not
available when the acquisition action is authorized. It is quarantined and
resolved after fulfillment, then admitted only if the prospective aggregate
resource--capability graph satisfies the current acquisition envelope.

Governing Dynamic Capabilities binds tool manifests and verifiable interaction
evidence \cite{zhou2026dynamic}; ACNBP defines capability discovery,
negotiation, attestation, and binding among agents \cite{huang2025acnbp}.
The authorization layer complements that work with post-fulfillment activation: its
strict separation keeps the tool or remote-agent declaration unchanged while
the provider-created resource instance, controller, actual manifest, or
aggregate graph differs.

Public OpenPort and intent-governed authorization work place normalized effect
decisions outside the model and bind them to intent
\cite{zhu2026openport,zhu2026igac}. The authorization layer supplies a prior certificate
interface: it decides whether an acquired resource may enter the set from
which such an operation-time gate accepts effects.

\subsection{Payment, commerce, and resource budgets}

AP2 and UCP define agent-mediated commerce artifacts and lifecycle structures
\cite{ap22026spec,ucp2026}. Formal and empirical analyses address payment
protocol consistency, mandates, context binding, replay, and threat surfaces
\cite{jiang2026formalpayments,aviv2026beyond,lan2026zerotrustpayments}.
Our architecture requires these controls where applicable. Its counterexample
holds transaction authorization and correct fulfillment fixed, then varies the
actual semantic authority of the delivered resource.

A recent systematization places transaction authorization within the broader
cross-layer security problem of agentic commerce
\cite{mao2026agenticcommerce}. Whisper-attack measurements further show that
an AP2-valid signed transaction can encode a shopping decision steered by
merchant text; A-VIP binds credential lookup, cart contents, and spending to
signed intent \cite{louck2026whisper}. These works govern the formation and
authorization of the transaction. Our analysis holds transaction authorization
and correct fulfillment fixed, then decides whether the semantic authority of
the delivered resource may enter the active graph.

Resource-bounded agent contracts model multidimensional limits and recursive
allocation \cite{ye2026agentcontracts}. Numeric and affine resource accounting
is complementary to the relational graph predicate. A free account can create
authority without spending a unit, a refund can restore money while leaving a
credential live, and a co-possession rule can fail even when every scalar
limit holds.

\subsection{Cloud identity and autonomous replication}

Native cloud systems already protect sensitive resource--identity bindings:
AWS constrains which role may be passed to an instance, and Google Cloud
requires permission to act as an attached service account
\cite{aws2026passrole,google2026actas}. The proposed architecture does not supersede native
IAM. It uses native decisions and authoritative inspection as evidence, then
enforces one cross-resource activation invariant for the registered common
profile.

RepliBench decomposes autonomous replication into obtaining resources,
accessing model artifacts, deployment, and persistence
\cite{black2025replibench}; self-improving-agent work demonstrates evaluated
successor creation in sandboxed research settings \cite{zhang2025dgm}. These
works motivate resource and descendant acquisition as observable security
events. This work contributes the access-control semantics, not a new
replication capability benchmark.

\section{Assurance Boundary}
\label{sec:assurance}

The proposed architecture makes a strong in-domain claim: no registered resource becomes
executable authority unless its actual resolved capability and complete
prospective acquisition graph satisfy every live envelope in its current
authority domain. The claim is
conditional on the explicit assumptions in Table~\ref{tab:assumptions}; the
conditions identify what an adapter and deployment must establish.
Table~\ref{tab:evidence-meaning} then separates the conclusion supported by
each evidence layer from the claims that still require a stronger deployment
or a different experiment.

\subsection{What each evidence layer establishes}

\begin{table}[H]
\caption{Interpretation of evidence layers.}
\label{tab:evidence-meaning}
\small
\begin{tabularx}{\textwidth}{L{0.20\textwidth} Y Y}
\toprule
Evidence & Establishes & Does not by itself establish \\
\midrule
Formal proof & all reachable accepted traces preserve the stated invariants under A1--A11 & correctness of a concrete resolver, identity service, provider, or deployment premise \\
Reference semantics & registered fixtures instantiate the formal state, exact permit binding, currentness, and slot obligations under the frozen finite profile & a machine-checked forward simulation for every implementation state, provider behavior, or production-scale performance \\
Independent trace checker & emitted flat-trace records satisfy registered finite event, state, and aggregate invariants without reusing runtime decision code & arbitrary hypergraph derivations or completeness of events that bypass the declared boundary \\
Runtime source/profile audit & frozen adapter captures the stated final logical acquisition request within explicit exclusions & an upstream patch, raw network-wire coverage, or unknown execution paths \\
Frozen upstream MCP client-component execution & both frozen clients complete initialize, discovery, stdio call, server observation, and response for all registered cases & a complete CLI session, model-driven selection, or direct execution of the Codex higher-level capture point \\
Local container integration & real create/inspect/certificate/re-inspect/start ordering and configuration checks on the registered host profile & durable activation-record atomicity, arbitrary cloud IAM, kernel compromise, or containment of an unsafe image \\
Upstream-observed composition & server-observed calls from both frozen clients enter one staged pipeline and reach either the registered rejection boundary or authorized completion; 12 cases reach Docker quarantine and four authorized first starts reach effect-permit commit and brokered dispatch & simultaneous model-to-provider execution, durable cross-system atomicity, engine-internal finality, or paths that bypass the broker \\
External-schema closure & public records supply, derive, or lack each required evidence field & real-world attack prevalence or undocumented provider semantics \\
\bottomrule
\end{tabularx}
\end{table}

Definition~\ref{def:implementation-refinement} states the full abstraction and
step-refinement obligation. The executable rows in
Table~\ref{tab:evidence-meaning} test named observable projections and
rejection boundaries of that obligation; they are not promoted to a universal
refinement proof for unexecuted states or undeclared deployment paths.

\subsection{Fail-closed profile admission}

A provider profile is admitted only when every policy-relevant output field is
authenticated, independently observable, or conservatively upper-bounded. If
the provider can create a raw credential directly in an agent-controlled
channel, quarantine is not complete. If aliases cannot be resolved, split
non-evasion is not available. If terminal state cannot be queried after a
timeout, the output remains indeterminate and quarantined. If effects can
bypass the broker, confinement does not cover those effects.

A conforming effect adapter also exposes a linearization primitive: it must
re-resolve the current lineage and atomically validate and consume the exact
effect permit with effect acceptance, or let the recipient redeem that permit
before applying an idempotent effect. A check followed by an unbound external
call is not a conforming implementation of the effect-confinement theorem.

These outcomes are explicit conformance failures, not inferred success. A
deployment may still use a narrower profile whose observable fields and
mediation points satisfy the contract.

\subsection{Scope of the guarantee}

The guarantee covers registered resource-to-authority transitions under
authenticated or independently observed evidence and complete mediation.
General right-acquisition safety, hidden real-world consequences, provider
behavior beyond registered evidence, objective alignment, side-channel
resistance, and partition liveness remain separate assurance domains. AP2,
UCP, A2A, MCP, OAuth, and cloud IAM contribute authenticated inputs within
their respective protocol scopes; the authorization layer composes those inputs into the
post-fulfillment activation decision.

The executable evaluation covers provider-free traces, inert credentials, and
local containers under the stated profiles. Performance results bind the named
host, implementation, and denominator. Finite zero observations characterize
those registered executions, while the conditional proof supplies the
universal in-model result.

\subsection{Operational consequence}

The assurance boundary produces a simple deployment rule: transaction success
may advance commerce state, and provider delivery may advance delivery state,
but only a current activation record may advance authority state.
No adapter may collapse those transitions. That rule remains meaningful even
when an organization chooses different payment, identity, runtime, or provider
systems.

\section{Conclusion}
\label{sec:conclusion}

An authorized acquisition is not yet authorization to exercise what it
returns. Autonomous agents make this distinction consequential because one
valid call can create a durable credential, execution environment, service
entitlement, or new principal whose future effects outlive the transaction.

The proposed architecture places the missing decision after fulfillment and before first
activation. It represents actual provider outputs in a typed acquisition
hypergraph, evaluates a relational downward-closed envelope over canonical
identities and aggregates, keeps returned resources in quarantine, commits
each finite-range output ordinal into at most one slot-unique active record,
revalidates the exact row and currentness before publishing its opaque handle,
preserves capability occupancy across refunds, and
rechecks current lineage before issuing a single-use effect permit that is
atomically consumed at effect commit. The resulting theorems cover
backed authority, non-amplification, split non-evasion, crash safety, freshness,
and end-to-end effect confinement within a registered decidable profile.

The executable evidence realizes that profile across five resource classes:
20 benign traces complete and 40 strict unsafe traces reject over 810 events,
with the latter divided among 25 activation, five delivery, and ten effect
rejections. Two runtime mappings agree on 60 logical cases, all 40
implementation-refinement traces agree under independent reconstruction, and
all 89 registered checker mutations reject, including field-complete sets for
23 acquisition-permit, 25 activation-permit, 16 effect-permit, and ten
effect-receipt fields. The verifier checks 118 acquisition-permit and 81
activation-permit bindings, 103 resolved rows, 109 output associations, 71
committed-row handle bindings, 30 canonical effect episodes, 117 source
reservations, five destruction receipts, five distinct target-to-domain
bindings, five same-ordinal replay rejections, and four logical-time advances.
All 15 activation cuts, five effect cuts, 32 schedules, 32 effect replays, 192
receipt substitutions, seven activation-recovery probes, five
post-commit/pre-publication drifts, eight contract-security probes, ten
effect-security probes, and 16 late completions satisfy their registered
safety outcomes. A local container gate passes all 20 cases.
Frozen Codex and Gemini MCP client components complete 27 calls each against
byte-identical deterministic local stdio witnesses. An 18-case
staged composition admits their server-observed calls and advances each to its
registered rejection boundary or authorized completion: both benign paths
complete, all 16 unsafe attempts issue no unauthorized Docker start request,
four authorized first starts occur, and all 12 created containers are removed.
This evidence exercises frozen client components and broker-mediated local
Docker starts; it is not a complete interactive CLI or model-driven session, a
network-provider experiment, or a production deployment.
The 32-unit public-source audit further establishes that transaction,
communication, tool, and replication structures contribute useful evidence
without collapsing into the separate activation certificate.

The separation is architectural: payment, budget, endpoint, identity, tool,
and provider-native controls retain their full roles, while none is silently
treated as the activation certificate for a newly acquired resource. This
gives heterogeneous autonomous-agent systems one auditable rule for moving
from acquisition to authority.

\appendix
\section{Proofs}
\label{app:proofs}

This appendix makes the transition-preservation argument explicit. We reason
over finite traces generated by the accepted transition relation in
Section~\ref{sec:formal}. A rejected transition leaves the security state
unchanged. Cryptographic verification, durable atomicity, canonical identity,
and conservative provider resolution are invoked only through A1--A11.

\subsection{Auxiliary definitions}

\begin{definition}[Current derivation]
\label{def:current-derivation}
A current derivation \(\Gamma\leadsto_s x\) is an acyclic rooted sub-hypergraph
ending at \(x\) such that:
\begin{enumerate}
  \item its root is \(\Gamma\in\mathsf{Live}_{\rho(\Gamma)}(s)\), with no
  implicit union of predecessor roots;
  \item every leaf belongs to \(\mathsf{Base}_\Gamma\): it is authenticated
  source evidence or a capability seed exactly authorized by \(\Gamma\), and
  no leaf is an active capability rooted elsewhere;
  \item every edge \(e\) has committed root \(\Gamma_e=\Gamma\),
  \(\kappa_e\in K_\Gamma\),
  \(\omega_e\in\Omega_{\kappa_e}\),
  \(I_e\in\mathcal I_{\kappa_e}\), and
  \(O_e\in\mathcal O_{\kappa_e}\), and its inputs and outputs equal their
  committed canonical resource identities and actual manifests;
  \item every edge satisfies \(R_{\kappa_e}(I_e,O_e)\) and its registered
  ordinal-specific \(F_{\kappa_e}(I_e,\Gamma,\omega_e)\) bound;
  \item every bound grant, policy, identity, contract, resource, and revocation
  epoch equals current authoritative state, and every node's descriptor
  validity interval contains \(\mathsf{now}(s)\);
  \item every non-source capability node \(y\), including each capability
  input, has a matching durable row \(a_y\) satisfying
  \(\mathsf{CurrentRow}_s(a_y,y)\) and
  \(A_s(\mathsf{slot}(a_y))=\mcode{active}\); and
  \item its terminal node has such a row \(a_x\), and its published handle
  \(\chi_x=\mathsf{handle}(a_x)\) satisfies
  \[
    \mathit{Handles}_s(\chi_x)=a_x.
  \]
\end{enumerate}
\end{definition}

\begin{definition}[Accepted trace]
An accepted trace is \(s_0\xrightarrow{a_1}s_1\cdots
\xrightarrow{a_n}s_n\), where \(s_0\) is empty except for trusted registries
and every \(a_j\) satisfies all guards of its named transition. External
messages that do not pass a transition guard may be recorded for audit but do
not change \(G\), \(\mathit{Active}\), \(\mathit{Handles}\),
\(\mathit{EffectReceipts}\), \(\mathit{UsedActivationKeys}\), or
\(\mathit{UsedEffectKeys}\). Logical time changes only through an accepted
\(\mathsf{AdvanceTime}\) step, whose post-state is the refresh-closed result;
there is no silent wall-clock transition.
\end{definition}

\begin{definition}[Implementation refinement obligation]
\label{def:implementation-refinement}
Let \(\mathcal S_{\mathrm{impl}}\) be an implementation state space and let
\(\mathcal A:\mathcal S_{\mathrm{impl}}\to\mathcal S\) be an abstraction
map.  An accepted implementation step \(j\to_{\mathrm{impl}}j'\) refines the
model when there are \(m\geq0\) formal steps
\[
 \mathcal A(j)=t_0\xrightarrow{b_1}t_1\xrightarrow{b_2}\cdots
 \xrightarrow{b_m}t_m=\mathcal A(j')
\]
where every \(b_i\) is an accepted formal transition,
and the implementation and formal observations agree at both endpoints:
\(\mathsf{Obs}_{\mathrm{impl}}(j)=
\mathsf{Obs}_{\mathrm{formal}}(\mathcal A(j))\) and likewise for \(j'\).
Here \(\mathsf{Obs}\) includes enabled handles, accepted protected effects,
canonical active authority, consumed slots, epochs, domain revisions, and
logical time. A zero-step match is permitted only when neither a security
observation nor authority state changes.

An executable specialization must map initial states to the formal initial
state, refine every accepted step, map a rejected step \(j\to j'\) to a
stutter satisfying \(\mathcal A(j)=\mathcal A(j')\), and
map crash recovery to the durable projection followed by formal recovery
steps.  These are the refinement obligations used to assess the project's
executable specializations; they do not assert a refinement proof for an
unexamined production deployment.
Consequently, an initial-state match plus stepwise accepted-transition or
stutter simulation transfers I1--I6 to every reachable implementation
abstraction.
\end{definition}

\begin{definition}[Crash and recovery]
\label{def:crash-recovery}
Let \(\mathsf{DurableProjection}(s)\) be the store state after the last
fully committed transaction represented in \(s\). It contains every write of
that transaction and its committed predecessors, and contains no uncommitted
or process-volatile write. Define
\[
  \mathsf{Crash}(s)=\mathsf{DurableProjection}(s).
\]
Let \(\mathcal R\) contain the ordinary guarded transitions that may re-drive
a durable outbox operation with its stable key, record or query a provider
terminal state, quarantine and resolve a correlated result, prepare or commit
activation or effect, and publish a handle. Then
\[
  \mathsf{Recover}(s)
  =\{t\mid t\in\mathsf{Reach}_{\mathcal R}(\mathsf{Crash}(s))\}.
\]
Every member is a finite prefix of recovery; the definition asserts neither
that recovery terminates, that an activation eventually succeeds, nor that an
effect is eventually accepted. A
published handle is a deterministic name for its committed active record, so
republication after recovery yields the same handle and creates no additional
active record. A \emph{fault-extended accepted trace} interleaves accepted
ordinary transitions, \(\mathsf{Crash}\) steps, and finite recovery prefixes.
\end{definition}

\begin{lemma}[Storage-derived slot uniqueness]
\label{lem:slot-uniqueness}
Under A6, every state in a fault-extended accepted trace contains at most one
committed \(\mathit{UsedActivationKeys}\) row and at most one committed
\(\mathit{Active}\) row
for each activation slot. Every committed active row and its corresponding
used-key row enter durable state in the same activation transaction. Under A6
and A10, at most one accepted effect receipt exists for each effect slot, and
that receipt and its \(\mathit{UsedEffectKeys}\) row enter durable state with the same
effect acceptance.
\end{lemma}

\begin{proof}
Order committed transactions by the serialization order supplied by A6.
Before the first relevant transaction, the claims are vacuous. An
activation transaction for slot \(j\) either aborts, changing no durable row,
or commits both proposed rows as one write set. If a row for \(j\) already
exists, the
relevant \(\mathsf{UNIQUE}(\mathit{slot})\) constraint rejects the conflicting write
and the complete transaction aborts. Thus induction over the serialization
order gives both cardinality bounds and row co-commitment. A crash removes
only uncommitted or volatile writes and retains the complete prefix through
the last commit. Recovery reuses the same slot and invokes only ordinary
guarded transitions; any conflicting activation transaction again aborts.
The same serialization argument applies to \(\mathsf{CommitEffect}\): A10
places permit validation, effect-slot consumption, receipt creation, and
effect acceptance at one registered linearization point, while the used-slot
uniqueness constraint rejects a second commit for that slot before another
effect can be accepted. Therefore neither crash nor recovery can invalidate
either bound.
\end{proof}

\begin{lemma}[Initial safety]
\label{lem:initial}
I1--I6 hold in \(s_0\).
\end{lemma}

\begin{proof}
There is no grant, prepared permit, active node, published handle, acquisition
edge, or accepted effect. I1--I4 and I6 are therefore vacuous, and there is no
live authority domain over which I5 must be checked. Grant registration later
requires the candidate domain's current projection to satisfy every predicate
in its candidate live-grant set. In particular, the first grant for a domain
must admit the empty graph. A utility condition belongs to the separate goal
predicate and does not alter safety.
\end{proof}

\begin{lemma}[Non-activation steps cannot create authority]
\label{lem:nonactivation}
Grant issuance, logical-time advancement, proposal, source reservation,
dispatch, quarantine, actual capability resolution, activation or effect
preparation, effect commit, refund, and provider-state reconciliation do not
make a previously nonexecutable node executable or publish a handle.
\end{lemma}

\begin{proof}
By transition definition, these operations may add grants, proposals,
reservations, outbox records, receipts, quarantined or resolved vault nodes,
or a prepared permit. Effect commit may add an effect receipt but no
acquisition authority. None of these transitions adds a durable active-node
row or a handle, and A1 and A10 make an exact current row plus broker
validation necessary for execution. \(\mathsf{AdvanceTime}\) may instead
remove executability through its atomic refresh.
Preparation binds a prospective graph root but does not append its edge to the
active graph. Refund changes commerce state only. Therefore the executable
set cannot grow.
\end{proof}

\begin{lemma}[Activation guard preserves I1--I6]
\label{lem:activation-guard}
If I1--I6 hold in \(s\) and
\(\mathsf{CommitActivation}(p_x,e)\) succeeds, they hold for the new edge and
output in \(s'\).
\end{lemma}

\begin{proof}
Let \(\Gamma=\Gamma_e\) be the candidate edge's committed provenance root.
The commit transaction verifies exact equality of the canonical resource,
actual-manifest root, descriptor, edge root, prospective graph root,
controller, beneficiary, grant-envelope digest, versioned provider profile,
source-evidence digest, contract version, security-state version, and epoch map bound
by \(p_x\); this establishes I1. The transaction writes the active record
before the handle can be published, establishing I2 for its singleton output.

The commit repeats the guard that \(\Gamma\in
\mathsf{Live}_{\rho(\Gamma)}(s)\), \(\kappa_e\in K_\Gamma\),
\(\omega_e\in\Omega_{\kappa_e}\), the input and output tuples are well typed,
\(\mathsf{now}(s)<\mathsf{expiry}(p_x)\), and every input is an authenticated source
authorized by the provenance root or a current executable capability derived
from that same root. The \(\mathsf{ActivationFresh}\) and acyclicity guards
ensure that the exact correlated, resolved vault output has no prior edge,
active row, or admitted slot and cannot become its own ancestor. The single-root rule rejects every
cross-root merge in the base transition system. Current-epoch,
root-currentness, and descriptor-validity checks together
with canonical identity resolution establish I3.
The contract resolver is exact, and the commit repeats
\(R_{\kappa_e}(I_e,\{x\})\) and
\(\alpha_x\restrict F_{\kappa_e}(I_e,\Gamma,\omega_e)\), establishing I4.

Let \(\rho=\rho(\Gamma)\). The candidate edge does not alter the live-grant
set. The transaction independently reconstructs \(\mathsf{Live}_\rho(s)\)
and the complete candidate projection, and accepts only if
\[
  \mathsf{Safe}_\rho(s,\Project^+_\rho(s,e)).
\]
This check includes every active subgraph whose provenance root is any live
grant in the domain, not just the candidate's \(\mathit{gid}\). Other domains are
unchanged. The atomic graph append establishes
\(\Project_\rho(s')=\Project^+_\rho(s,e)\), so I5 holds after commit even
when several fresh grant identifiers share \(\rho\).

Membership of \(p_x\) in the immutable permit ledger was established at
preparation. The proposed \(\mathit{UsedActivationKeys}\) consumption row,
graph append, active record, and domain-revision advance are one transaction
write set; the permit itself remains unchanged.
Because the activation succeeds, A6 places that complete write set at one
durable linearization point; an abort would have changed none of it.
Lemma~\ref{lem:slot-uniqueness} derives that no other committed active record
can occupy the same slot. The transition accepts no effect and does not alter
any prior effect receipt, so I6 is preserved. Previously prepared effect
permits may become stale after the domain-revision advance, but remain exactly
bound objects under I1 and cannot commit under the effect guard.
\end{proof}

\begin{lemma}[Other accepted transitions preserve I1--I6]
\label{lem:preservation}
If I1--I6 hold in \(s\), every accepted transition other than
\(\mathsf{CommitActivation}\) preserves them.
\end{lemma}

\begin{proof}
For \(\mathsf{IssueGrant}(\Gamma)\), let
\(s^+=s\oplus\Gamma\) and \(\rho=\rho(\Gamma)\). If the grant is current in
\(s^+\), the guard evaluates the candidate live-grant set, including the new
profile, on the existing shared-domain projection and accepts only if
\[
  \mathsf{Safe}_\rho(s^+,\Project_\rho(s^+)).
\]
If it is not yet current, \(\mathsf{Live}_\rho\) and the projection are
unchanged. In both cases registration has separately established
\(\Phi_\Gamma(\varnothing)=\mathsf{true}\). A fresh \(\mathit{gid}\) in an
existing domain therefore adds a live conjunct when current rather than an
empty accounting compartment. Other domains are unchanged, so I5 is
preserved. The domain-revision advance makes older effect permits stale but
changes neither their exact bindings nor any historical accepted effect.

For \(\mathsf{AdvanceTime}(\theta')\), the clock write and
\(\mathsf{Refresh}\) have one serialization point. Expired roots,
descriptors, and their dependent descendants are fenced, including committed
active-authority rows whose handles have not yet been published, so no
remaining active-authority row loses currentness and no executable node loses
I3. Removing an expired live grant weakens the
conjunction; the associated fencing only decreases the projection, so
downward closure preserves every remaining conjunct. At a grant onset, the
new conjunct is tested on the current projection. If it holds, I5 holds
directly; if it fails, refresh fences that domain, and every live grant accepts
the resulting empty projection. Thus the post-state is refresh-closed and
I2--I5 hold. The step creates no permit or effect, so I1 and historical I6 are
unchanged.

\(\mathsf{PrepareActivation}\) checks and signs every binding in
Definition~\ref{def:activation-permit}; hence the new prepared permit satisfies
I1, while I2--I6 are unchanged. \(\mathsf{PrepareEffect}\) re-resolves one
exact current row \(a\) for \(x\), sets
\(\rho=\rho(\mathsf{root}(a))\), normalizes the complete requested episode,
checks \(\mathit{Handles}_s(\chi)=a\), the active slot-level authority
coordinate, \(\mathsf{CurrentRow}_s(a,x)\), \(\mathsf{Exec}_s(x)\), the instance descriptor,
the current target-to-domain map commitment, \(\mathsf{EpisodeCurrent}_s(u)\),
\(\mathsf{Obl}_s(o_{\alpha_x},u)\), and the current shared-domain predicate,
and signs every field in Definition~\ref{def:effect-permit}. It therefore establishes
I1 for the immutable permit appended to \(\mathit{EffectPermits}\), but
accepts no effect and changes none of I2--I6. The
remaining pre-commit transitions in Lemma~\ref{lem:nonactivation} neither
create a prepared permit nor change an existing permit binding.

For \(\mathsf{CommitEffect}(p_u)\), serializability places the effect commit
either before or after every concurrent authority-changing transaction. Let
\(s^-\) be the unique pre-state at its linearization point. The guard resolves
the permit's handle to the same active row \(a\) and instance \(x\), sets
\(\rho=\rho(\mathsf{root}(a))\), and re-resolves the complete epoch vector,
domain revision, normalized episode, and lineage. It checks exact equality
with \(p_u\), including the current target-to-domain map commitment,
\(\mathit{Handles}_{s^-}(\chi)=a\),
\(A_{s^-}(\mathsf{slot}(a))=\mcode{active}\),
\(\mathsf{CurrentRow}_{s^-}(a,x)\), \(\mathsf{Exec}_{s^-}(x)\),
\(\mathsf{now}(s^-)<\mathsf{expiry}(p_u)\),
\(u\in\sem{\alpha_x}\), \(\mathsf{EpisodeCurrent}_{s^-}(u)\),
\(\mathsf{Obl}_{s^-}(o_{\alpha_x},u)\), and
\(\mathsf{Safe}_\rho(s^-,\Project_\rho(s^-))\), and requires an unused
effect slot and membership of \(p_u\) in the immutable permit ledger. Under
A6 and A10, appending the used-slot consumption row and
accepted receipt with its witness, and accepting the protected effect form one
indivisible commit. A
binding substitution, intervening epoch or revision change, or replay
therefore aborts before the effect. A successful commit establishes I6 and
the effect-slot part of Lemma~\ref{lem:slot-uniqueness}; it changes no grant,
active graph, activation record, or handle, so the first five invariants remain
true and the just-established I6 completes preservation of I1--I6.

\(\mathsf{PublishHandle}\) requires the exact durable row \(a\),
\(\mathsf{CurrentRow}_s(a,\mathsf{inst}(a))\), and an active authority
coordinate before atomically installing the deterministic, collision-checked
\(\mathit{Handles}(\chi)=a\) mapping. The equal mapping is idempotent and an
unequal mapping aborts, so publication cannot make a stale or different
committed row executable and preserves I1--I6.

\(\mathsf{Fence}\) advances an epoch and invokes the same atomic refresh used
for time. It cannot widen a descriptor or add an active node, and refresh
fences stale committed rows even before handle publication while preserving
current derivations for every remaining executable node. Each
affected domain's executable projection can only decrease; if a live-grant
set changes, refresh applies the same direct-check-or-empty-domain rule as the
time case. Downward closure and empty-graph admissibility therefore preserve
I5.
\(\mathsf{DestroyVerified}\)
removes the destroyed instance from the active projection only after current
authoritative evidence binds the same canonical resource, while retaining its
immutable provenance. Its guard atomically fences every descendant whose
registered derivation requires the destroyed resource. Thus every remaining
executable descendant keeps a current derivation, and the resulting normalized
graph is below the original in \(\preceq_s\); the downward closure of every
predicate in each affected domain gives I3--I5. I1 is unchanged, removing the
destroyed handle preserves I2, and I6 records what held at each historical
effect linearization point rather than asserting that its authority remains
current forever.

\(\mathsf{RefundSettled}\) changes commerce evidence but not the active graph,
and \(\mathsf{ReconcileLateResult}\) adds at most a quarantined output. A late
result must execute the ordinary resolution and activation sequence before
joining \(\mathit{Active}\). Therefore every non-activation transition preserves all
six invariants.
\end{proof}

\begin{lemma}[Fault steps preserve safety]
\label{lem:fault-preservation}
If I1--I6 hold before a crash, they hold in
\(\mathsf{Crash}(s)\) and throughout every finite recovery prefix in
\(\mathsf{Recover}(s)\). The slot bounds of
Lemma~\ref{lem:slot-uniqueness} also hold throughout.
\end{lemma}

\begin{proof}
By A6, \(\mathsf{Crash}(s)\) is the state after the last fully committed
transaction, hence a committed prefix of the pre-crash execution rather than
a partial authority-changing write. I1--I6 held at that prefix by the ordinary
transition lemmas. Every recovery action is an ordinary guarded transition in
\(\mathcal R\): Lemma~\ref{lem:activation-guard} covers a successful
activation transaction and Lemma~\ref{lem:preservation} covers the remaining
actions. Induction over the finite recovery prefix preserves I1--I6.
Lemma~\ref{lem:slot-uniqueness} separately supplies the slot bounds across
both the crash projection and those transactions. In particular, a crash
before an effect commit's A10 linearization point exposes neither acceptance
nor its durable receipt and used-slot row; a crash after that point exposes
the complete committed set. Recovery cannot create a second acceptance for
the same effect slot.
\end{proof}

\begin{theorem}[Invariant preservation]
\label{thm:invariant-preservation}
Under A1--A11, I1--I6 hold after every prefix of every
fault-extended accepted trace.
\end{theorem}

\begin{proof}
By induction on trace length. Lemma~\ref{lem:initial} is the base case.
Lemma~\ref{lem:activation-guard} handles a successful activation step, and
Lemma~\ref{lem:preservation} handles every other accepted ordinary
transition. Lemma~\ref{lem:fault-preservation} handles crashes and finite
recovery prefixes.
\end{proof}

\subsection{Separation propositions}

\begin{proof}[Proof of Proposition~\ref{prop:observation}]
Assume for contradiction that a function \(f(z)\) is sound and complete for
activation. Let provider worlds \(w_1,w_2\) produce the same transaction tuple
\(z\), but resolve to \(\alpha_1,\alpha_2\), inducing candidate edges
\(e_1,e_2\) with
\(\Phi_\Gamma(H\boxplus_{\rho(\Gamma)}
(\{e_1\}\cup\mathsf{out}(e_1)))=\mathsf{true}\) and
\(\Phi_\Gamma(H\boxplus_{\rho(\Gamma)}
(\{e_2\}\cup\mathsf{out}(e_2)))=\mathsf{false}\). Because its input is equal,
\(f\) returns the same result in both worlds. If that result accepts, it is
unsound in \(w_2\); every non-accepting result, including indeterminate, is
incomplete in \(w_1\). Contradiction. An
authenticated actual-output observation or a sound profile upper bound is
therefore necessary to distinguish the worlds.
\end{proof}

\begin{proof}[Proof of Proposition~\ref{prop:correlation}]
Fix one canonical control root \(\mathit{cr}\) and two descriptors \(\alpha_r\) and
\(\alpha_p\) containing, respectively, a sensitive-read and an
external-publish episode. Let \(\Phi_\Gamma(H)\) hold exactly when \(H\) does
not contain both capabilities controlled by \(\mathit{cr}\). This predicate is
downward closed. Its component allowlists contain \(\mathit{cr}\), both effect classes,
both targets, and all other individual fields, so a componentwise membership
test accepts each descriptor and their field union. Yet the graph containing
both makes \(\Phi_\Gamma\) false. Hence componentwise checks do not preserve
the correlation.
\end{proof}

\begin{proof}[Proof of Proposition~\ref{prop:zero-cost}]
Let the monetary invariant be total spend at most \(B>0\), and let
\(\Phi_\Gamma\) permit at most one active descendant for \(\mathit{cr}\). Two
zero-price enrollment operations preserve total spend \(0\le B\). If both
outputs activate, the descendant count is two and \(\Phi_\Gamma\) is false.
Thus monetary compliance does not imply acquisition safety. The same
construction applies to prohibited co-possession.
\end{proof}

\subsection{Safety theorem proofs}

\begin{proof}[Proof of Theorem~\ref{thm:quarantine}]
Lemma~\ref{lem:nonactivation} enumerates every pre-commit transition and shows
that none writes active authority or publishes a handle. A1 ensures no
registered acquisition or effect path bypasses these transitions; A3 prevents
resolution from manufacturing a silent grant; A6 orders the durable active
record before publication; and A10 requires a broker-accepted handle for a
protected effect. Therefore only a successful
\(\mathsf{CommitActivation}\) creates active authority, and executability
arises only when the broker subsequently publishes a handle for that durable
active record.
\end{proof}

\begin{proof}[Proof of Theorem~\ref{thm:backed}]
Consider executable \(x\). By the definition of \(\mathsf{Exec}_s\), it has a
unique matching active row \(a\); set \(\Gamma=\mathsf{root}(a)\). By I2 the
row has a committed activation edge. Invariant preservation gives I3 directly:
the row has a current acyclic derivation whose unique provenance root is the
live \(\Gamma\). Finiteness and acyclicity make recursive traversal terminate
at a source or capability seed authorized by that root. Because the base
transition system rejects cross-root capability combinations, every traversed
edge retains the same root. Hence the derivation terminates at exactly one live
root \(\mathit{gid}\) under the profile's root-identity rule. This
provenance uniqueness does not create a separate accounting compartment: I5
still evaluates the union of all live roots in that root's authority domain.
\end{proof}

\begin{proof}[Proof of Theorem~\ref{thm:nonamp}]
Theorem~\ref{thm:invariant-preservation} gives, for every live authority
domain \(\rho\),
\(\mathsf{Safe}_\rho(s,\Project_\rho(s))\) in every reachable state. This is
the relational non-amplification statement, including the case in which the
domain contains several provenance \(\mathit{gid}\) values. For the root-specific
episode-set corollary, every \(x\in\mathsf{Active}_\Gamma(s)\) lies on a
valid current path from
\(\Gamma\) by Theorem~\ref{thm:backed}. I4 supplies
\(\kappa_e\in K_\Gamma\), \(\omega_e\in\Omega_{\kappa_e}\), well-typed
input/output tuples, \(R_{\kappa_e}\), and the registered ordinal-specific
per-edge bound. The source leaves belong to \(\mathsf{Base}_\Gamma\).
Induction on path length, with every input tuple's components already
reachable, applies the \textsc{Step}
rule defining \(\mathsf{Reach}_\Gamma\), so every such \(x\) belongs to that
least set and \(\sem{\alpha_x}\subseteq\mathsf{Cl}_\Gamma\) by definition.
Unioning these inclusions proves the displayed subset. The main I5 statement
remains strictly stronger because the conjunction over
\(\mathsf{Live}_\rho(s)\) also constrains identities, counts,
co-possession, and topology across sibling grant identifiers.
\end{proof}

\begin{proof}[Proof of Theorem~\ref{thm:split}]
By Definition~\ref{def:projection}, normalization removes precisely the
syntactic distinctions allowed by the theorem and retains every semantic
output, canonical identity, count, relation, and edge. Equal normalized
additions therefore present identical shared-domain projections to every
predicate in \(\mathsf{Live}_\rho(s)\), so the deterministic conjunction
\(\mathsf{Safe}_\rho\) gives equal values. In particular, placing a second
worker under a fresh sibling \(\mathit{gid}\) does not hide it: both workers remain in
\(\Project_\rho\), and a domain clause allowing at most one active worker
rejects the second activation.

For the execution claim, fix an interval in which the live-grant set is
unchanged and no exogenous transition removes authority from the projection,
and suppose an interleaving fully activates a \(\Delta\) for which
\(\mathsf{Safe}_\rho(s,H\boxplus_\rho\Delta)\) is false. Let \(e_j\) be the
final fragment commit and \(s^-\) its pre-state. Earlier fragments may have
advanced the domain revision, so \(e_j\) must have been prepared again or
revalidated against that current revision. Because no authority was removed,
\(\Project^+_\rho(s^-,e_j)\) contains the normalized initial \(H\), all of
\(\Delta\), and possibly additional authority. Downward closure gives the
contrapositive: since the smaller \(H\boxplus_\rho\Delta\) is unsafe, this
candidate extension is also unsafe. The commit guard therefore rejects
\(e_j\), contradicting successful full activation. A safe prefix may remain active,
which is consistent with non-evasion. If an accepted time, epoch,
grant-currentness, fence, or destruction transition intervenes, it advances
the domain revision and atomically refreshes the projection. No stale permit
can cross that step. The argument then restarts from the new refresh-closed
projection and live-grant set, so no unsafe intermediate aggregate is exposed.
\end{proof}

\begin{proof}[Proof of Theorem~\ref{thm:crash}]
Consider an arbitrary state \(s\) of an arbitrary fault-extended accepted
trace and an arbitrary executable instance \(x\) in that state. By the exact
execution definition there is a unique row \(a\) satisfying
\(\mathsf{CurrentRow}_s(a,x)\) and the handle and active-coordinate guards;
put \(j=\mathsf{slot}(a)\) and
\(\Gamma=\mathsf{root}(a)\). Lemma~\ref{lem:slot-uniqueness} gives at most one durable active row for
\(j\) before any fault and preserves that bound through every crash and finite
recovery prefix. More explicitly, a crash before the activation transaction's
linearization point projects to a state containing none of its writes. A
crash after that point retains its complete write set, including the consumed
slot and active row. There is no durable state containing only one of those
two rows. Handle publication occurs only from the committed row. If recovery
repeats publication, determinism yields the same handle rather than a second
authority.

A replayed activation permit addresses the same slot, so its conflicting
transaction aborts under \(\mathsf{UNIQUE}(\mathit{slot})\). A duplicate receipt correlates with
the existing logical operation and remains quarantined unless it passes a
separate ordinary activation under the complete graph. A late success follows
the same rule. Revalidation within any successful activation commit
establishes the current manifest, live root, descriptor validity, epochs, and
\(\mathsf{Safe}_{\rho(\Gamma)}(s,\Project_{\rho(\Gamma)}(s))\);
each accepted time or epoch transition refreshes them before its post-state is
observable. Theorem~\ref{thm:invariant-preservation} preserves these facts
across the fault-extended trace. Consequently every
executable handle is backed by the unique safe committed active record for its
slot, and no slot produces two simultaneously executable authorities. This is
an at-most-once safety result: it does not assert that retries terminate or
that any attempted activation eventually commits. A later epoch advance may
leave a historical handle reference recorded while making
\(\mathsf{Exec}_s(x)\) false, which is consistent with the theorem's
current-state antecedent. Since the used-slot row and active row co-commit,
\(\mathsf{SlotConsumed}_s(a)\) also holds, proving every conjunct in the
theorem's implication.
\end{proof}

\begin{proof}[Proof of Theorem~\ref{thm:refund}]
By definition, refund changes only \(C_s(o)\) to
\(C_{s'}(o)=\mcode{refunded}\) and updates source evidence. It does not change
any \(D_s(o,\omega)\), any slot-level \(A_s(j)\), active row, or graph edge. Thus
\(\Project_\rho(s')=\Project_\rho(s)\) for every authority domain \(\rho\),
and every domain capacity clause has the same value immediately before and
after the refund. A later candidate
is evaluated with the first capability still present unless a separate
authorized transition occurs under the registered profile. Concretely, the
old capability must be fenced (optionally followed by ordinary activation
under a narrower descriptor) or verified as destroyed. Source restoration is
not evidence that either event occurred. Hence restored money cannot
resurrect a semantic slot by itself.
\end{proof}

\begin{proof}[Proof of Theorem~\ref{thm:epoch}]
Activation permits, active records, and effect permits bind the complete epoch
vector; effect permits also bind the current domain revision. Activation
commit, effect preparation, and effect commit compare these values with
authoritative state. After any bound epoch advances, exact equality fails: an
old activation permit cannot commit, an old effect permit cannot linearize an
effect, and an old handle cannot obtain a current replacement permit. An
authority-domain change likewise advances its revision and invalidates a
prepared effect permit. Re-resolution and new activation are the only
transitions that can bind a changed resource to the new epoch vector.
\end{proof}

\begin{proof}[Proof of Theorem~\ref{thm:confinement}]
Let \(r\) and \(p_u\) witness
\(\mathsf{AcceptedEffect}_{s_{k+1}}(x,u)\), with \(r\) first appearing in
\(s_{k+1}\). By A1, A10, and I6, there is exactly one transition
\(s_k\xrightarrow{\mathsf{CommitEffect}(p_u)}s_{k+1}\) that atomically
accepts the effect and appends \(r\) and its previously unused effect-slot row.
In the pre-state \(s_k\), the broker resolves the handle to the exact durable
row \(a\) and instance \(x\), with
\(\rho=\rho(\mathsf{root}(a))\) and
\(\mathit{Handles}_{s_k}(\chi)=a\),
\(A_{s_k}(\mathsf{slot}(a))=\mcode{active}\), and
\(\mathsf{CurrentRow}_{s_k}(a,x)\). I6 also gives
\(\mathsf{Exec}_{s_k}(x)\). By
Theorem~\ref{thm:backed}, the row has a complete current derivation there. The
guard compares the complete normalized episode, row root, authority domain,
target-to-domain map commitment, epoch vector, and domain revision with
\(p_u\), requires
\(\mathsf{now}(s_k)<\mathsf{expiry}(p_u)\), and checks
\(u\in\sem{\alpha_x}\), \(\mathsf{EpisodeCurrent}_{s_k}(u)\),
\(\mathsf{Obl}_{s_k}(o_{\alpha_x},u)\), and
\(\mathsf{Safe}_\rho(s_k,\Project_\rho(s_k))\).

If the episode, handle, or active record was substituted after preparation,
exact permit equality fails. If activation, fencing, destruction, grant
issuance, or an epoch change intervened, the epoch or domain-revision check
fails. If the same logical effect is replayed, the used-slot check fails. By
A10 each failure aborts before the protected effect, whereas success appends
the slot-consumption row atomically with effect acceptance. Consequently an
accepted effect lies within the actual capability and current shared-domain
envelope and cannot be detached from its acquisition provenance by a
prepare-to-commit race.
\end{proof}

\bibliographystyle{ACM-Reference-Format}
\setlength{\bibsep}{-0.2pt}
\bibliography{references}

\end{document}